\documentclass{article}

\usepackage{PRIMEarxiv}

\usepackage{natbib}

\usepackage{graphicx} 
\usepackage{subfig}
\usepackage{wrapfig}
\usepackage{xfrac}
\usepackage{subcaption}
\usepackage{tikz}
\usepackage[utf8]{inputenc} % allow utf-8 input
\usepackage[T1]{fontenc}    % use 8-bit T1 fonts
\usepackage{hyperref}       % hyperlinks
\usepackage{url}            % simple URL typesetting
\usepackage{booktabs}       % professional-quality tables
\usepackage{amsfonts}       % blackboard math symbols
\usepackage{nicefrac}       % compact symbols for 1/2, etc.
\usepackage{microtype}      % microtypography
\usepackage{xcolor}         % colors
\usepackage{bbm}
\usepackage{amsmath}
\usepackage{amsthm}
\usepackage{caption}
\usepackage{subcaption}
\usepackage{changes} % to track changes
\usepackage{amssymb}
\usepackage{mathtools}

\usepackage{hyperref,cleveref}       % hyperlinks
\newtheorem{theorem}{Theorem}[section]
\newtheorem{lemma}{Lemma}[section]
\newtheorem{corollary}{Corollary}[theorem]
\theoremstyle{definition}
\newtheorem{definition}{Definition}[section]
\theoremstyle{remark}

\newcommand{\KL}{\textsf{KL}}
\newcommand{\MDL}{\textnormal{MDL}}

\newcommand{\abs}[1]{\left\lvert{#1}\right\rvert}
\newcommand{\removed}[1]{}

\newtheorem*{restatethm}{\theoremname} 

\newcommand{\restate}[2]{%
    \def\theoremname{\textbf{\autoref{#2}}} % Makes "Lemma X.Y" or "Theorem X.Y" bold
}

\newcommand{\EmpiricalBayes}{}
\newcommand{\ProfilePosterior}{^{\textsf{mod}}}

\newcommand{\Qgreen}{\hat{Q}\EmpiricalBayes}
\newcommand{\Jgreen}{J_\lambda\EmpiricalBayes}
\newcommand{\QEB}{\Qgreen}
\newcommand{\JEB}{\Jgreen}

\newcommand{\Qblue}{\hat{Q}\ProfilePosterior}
\newcommand{\Jblue}{J_\lambda\ProfilePosterior}
\newcommand{\Jbluem}{J_{\lambda_m}\ProfilePosterior}
\newcommand{\QPL}{\Qblue}
\newcommand{\JPL}{\Jblue}

\newcommand{\QBayes}{\hat{Q}^{\textsf{Bayes}}}
\newcommand{\TJ}{\Tilde{J}}
\newcommand{\qstar}{q^*_m}
\newcommand{\qhat}{\hat{q}_m}
\DeclareMathOperator*{\argmax}{arg max}
\newcommand{\redcomment}[1]{} %{\textcolor{red}{#1}}
\newcommand{\source}{\mathcal{D}}
\newcommand{\achieve}{\mathcal{M}}
\newcommand{\entropicT}{T}
\newcommand{\natinote}[1]{\todo[color=blue!20]{{#1}}}

\title{
Overfitting and Generalizing with (PAC) Bayesian Prediction in Noisy Binary Classification
}

\author{
  \\
  \\
  \texttt{} \\
   \And
   \\
  \\
  \texttt{} \\
}

\author{
  Xiaohan Zhu\\
  The University of Chicago \\
  \texttt{xiaohanz@uchicago.edu} \\
  \And
  Mesrob I. Ohannessian\\
  The University of Illinois at Chicago \\
  \texttt{mesrob@uic.edu}
   \And
  Nathan Srebro \\
  Toyota Technological Institute at Chicago\\
  \texttt{nati@ttic.edu} \\
}

\begin{document}
\maketitle

\begin{abstract}
We consider a PAC-Bayes type learning rule for binary classification, balancing the training error of a randomized ``posterior'' predictor with its KL divergence to a pre-specified ``prior''. This can be seen as an extension of a modified two-part-code Minimum Description Length (MDL) learning rule, to continuous priors and randomized predictions. With a balancing parameter of $\lambda=1$ this learning rule recovers an (empirical) Bayes posterior and a modified variant recovers the profile posterior, linking with standard Bayesian prediction (up to the treatment of the single-parameter noise level). However, from a risk-minimization prediction perspective, this Bayesian predictor overfits and can lead to non-vanishing excess loss in the agnostic case.
  Instead a choice of $\lambda \gg 1$, which can be seen as using a sample-size-dependent-prior, ensures uniformly vanishing excess loss even in the agnostic case. We precisely characterize the effect of under-regularizing (and over-regularizing) as a function of the balance parameter $\lambda$, understanding the regimes in which this under-regularization is tempered or catastrophic. This work extends previous work by \citet{ZS} that considered only discrete priors to PAC Bayes type learning rules and, through their rigorous Bayesian interpretation, to Bayesian prediction more generally.
\end{abstract}

% keywords can be removed
% \keywords{}
\section{Introduction}
%%%%%%%%%%%%%%%%%%%%%%%%%%%%%%%%%%%%%%%%%%%%%%%%%%%%%%%%%%%%%%%%%%%%%%%%%%%%%%%%
%%%%%%%%%%%%%%%%%%%%%%%%%%%%%%%%%%%%%%%%%%%%%%%%%%%%%%%%%%%%%%%%%%%%%%%%%%%%%%%%

Consider the problem of supervised binary classification, where we seek to learn predictors $h: \mathcal X \to \{0,1\}$ using a labeled training set $S$ of $m$ samples drawn i.i.d. from a source distribution $\source$ over $\mathcal X \times \{0,1\}$. We primarily focus on the \emph{agnostic setting}, i.e., when $\source$ is arbitrary (See the detailed formalism in \autoref{formal setup}).

If hypotheses are chosen from a countable set, a classical paradigm is the modified two-part-code Minimum Description Length (MDL) learning rule, given by:
\begin{align}
    \MDL_{\lambda}(S) &= \underset{h:\mathcal{X}\rightarrow\{0,1\}}{\text{arg min }} \log \binom{m}{mL_S(h)} + \lambda \log \frac{1}{\pi(h)} \notag\\
    &\approx \underset{h:\mathcal{X}\rightarrow\{0,1\}}{\text{arg min }}  mH(L_S(h))  + \lambda \log \frac{1}{\pi(h)},\label{eq:mdl-intro}
\end{align}
where $L_S(h)$ is the (zero-one) training error, $H(\cdot)$ is the binary entropy function\footnote{We use $H(\alpha)=-\alpha \log \alpha-(1-\alpha)\log (1-\alpha)$, for $\alpha \in[0,1]$, to denote the entropy of a $\text{Ber}(\alpha)$ random variable. (All logarithms are base-$2$ and entropy is measured in bits.)}, and $\pi$ is a (discrete) prior over hypotheses.
% Similarly, we write $L_\source(h)$ for the population error of $h$ on a test sample from $\source$ and $L_\source(Q)= \mathbb{E}_{h\sim Q}\left[L_\source(h)\right]$ for the expected population error when $h$ is sampled from $Q$.
When $\lambda=1$, we obtain the original MDL rule, which we write as $\MDL_1$ and where the objective naturally describes a joint description length for the sequence of labels (or equivalently of errors) and the model. It can also be viewed as a Maximum A Posteriori (MAP) predictor, under an independent noise source model, $Y=h(X)\oplus N$ with $N\sim\text{Ber}(\eta)$, selecting $\hat h = \argmax_h \max_\eta \Pr\left(h\middle| S; \eta\right)$, with $\eta$ a nuisance parameter for the noise.\footnote{$\text{Ber}(\alpha)$ denotes a Bernoulli random variable with expectation $\alpha$ and $a \oplus b$  denotes the XOR of two bits $a,b\in\{0,1\}$.}
% selecting the predictor maximizing the posterior $\Pr\left(h\middle| S\right)$, under a data-generating model 
% Note that the original MDL rule is with $\lambda=1$, which we write as $\MDL_1$. 

How well do these rules perform? $\MDL_1$ was identified early on to be suboptimal and inconsistent in the agnostic setting \citep{GL}. The more recent work of \citet{ZS} interprets this as under-regularization or overfitting. They view $\lambda$ as a regularization parameter, show that $\lambda \to \infty$ is necessary to ensure consistency and that the ``correct'' choice is $\lambda_m \approx \sqrt{m}$. In contrast, any bounded $\lambda$ leads to harmful overfitting. They then precisely characterize the worst-case harm of such overfitting as a function of $\lambda$, through tight bounds on the \emph{worst-case limiting error} (defined in \autoref{continuous:main}). In particular, in the agnostic case, they show that $\MDL_{\lambda}$ enjoys tempered overfitting for $0 < \lambda < \infty$ and determine which scalings $\lambda_m$ lead to $\MDL_{\lambda_m}$ being consistent and which lead to catastrophic overfitting.

But what about uncountable hypotheses with continuous priors and learning rules that do better aggregation of hypotheses, beyond MAP? To handle large hypothesis spaces, instead of MDL, PAC-Bayes learning rules
%\natinote{I changed this to citep.  Please go over all refrences and change those that are not used as a noun in the sentence to citep}
\citep{mcallester2003pac, catoni2007pac} can be used. These  start with a (continuous) ``prior'' $\Pi$ over hypotheses and, instead of a single hypothesis, produce a ``posterior'' $Q$, based on balancing the training error with the KL divergence between $Q$ and $\Pi$:
% \begin{equation} \label{eq:generic-PAC-Bayes}
%     Q_{\lambda,f}(h) = \argmin_Q H(L_S(Q)) + \lambda \KL(Q \Vert \Pi),
% \end{equation}
\begin{align}\label{eq:PAC-Bayes-intro}
    \Qgreen_{\lambda}(S) = \underset{Q}{\text{arg min }} mH(L_S(Q)) +  \lambda \KL(Q \Vert \Pi).
\end{align}
where $L_S(Q) = \mathbb{E}_{h\sim Q}\left[L_S(h)\right]$ represents the expected training error if predictions are made by sampling $h$ from the posterior $Q$.\footnote{$\KL(Q \Vert \Pi) = \int \log (\frac{dQ}{d\Pi}) dQ$ denotes the KL divergence between two distributions $Q$ and $\Pi$.}  
% (See \autoref{formal setup} for a complete description). 

The standard form of PAC-Bayes uses a linear expression in $L_S(Q)$. However, using the entropy makes this variant a closer analog to MDL, taking us from countable hypothesis classes with discrete priors to uncountable hypothesis classes with continuous priors. The posterior takes the place of the learned model and we can interpret the KL divergence term as the complexity of the posterior. The PAC-Bayes learning rule has been extremely successful as both a direct algorithmic tool \citep[e.g.,][]{germain2009pac, germain2015risk, alquier2024user, laviolette2011pac, germain2013pac,pmlr-v26-seldin12a} and as an analytic device, e.g., for analyzing generalization in deep learning \citep[e.g.,][]{DziugaiteR17, nagarajandeterministic, neyshabur2018pac, london2017pac, NEURIPS2020_c3992e9a, NEURIPS2018_9a0ee0a9, NEURIPS2019_05e97c20}.

Can we understand how overfitting and generalization occur in PAC-Bayes in the same way as we understand them for MDL, in the agnostic setting? And can we extend these insights to more Bayesian learning rules? We ask and answer two sets of questions:
\begin{itemize}
    \item[(a)] Analogous to the $\MDL_{\lambda}$ learning rule, we ask, what is the worst-case limiting error of this rule? If there exists some competitor (i.e.~oracle) $Q^*$ with low population (test) error $L^*$ and low $\KL(Q^* \Vert \Pi)$, what is the limiting behavior of the loss $L_\source(\Qgreen_{\lambda_m})$ as a function of $\lambda_m$? Can we characterize the limiting error as $m\to\infty$ and the cost of overfitting along the entire regularization path, as a function of $\lambda$ and $L^*$?
    \item[(b)] How is the learning rule $\Qgreen_{\lambda}$, in particular with $\lambda=1$, related to a Bayesian posterior? In particular, does the MAP interpretation of $\MDL_1$ extend and, moreover, how can we interpret taking $\lambda \neq 1$ from a Bayesian modeling perspective?    
\end{itemize}

Our contributions are as follows:
\begin{itemize}

    \item We obtain an exact characterization of the worst case limiting error of $\Qgreen_{\lambda}$ as $m\rightarrow \infty$ via matching upper and lower bounds, for all $\lambda$ and $L^*$ (\Cref{cor:finite} and equation \eqref{ell}). For finite $\lambda$, we show tempered inconsistency/overfitting, whereas for $1\ll \lambda \ll m/\log m$, $\Qgreen_{\lambda}$ is agnostically consistent (\Cref{lambda_infty_coro}).
    
    For all $\lambda$ values and $\lambda_m$ scalings, we provide concrete finite sample upper bounds or uniform guarantees (\autoref{MDL UB} and \autoref{lambda infty}). 
    
    \item We show that the learning rule $\Qgreen_{\lambda}$ with $\lambda=1$ corresponds to an empirical Bayes procedure (\autoref{greenEB:lambda=1}) and that other choices of $\lambda$ are equivalent to an $m$-dependent prior on the source's noise model ($\eta$) (\autoref{greenEB:m_prior}). Dependence on sample size here is crucial, as any static prior would incur the same inconsistent behavior as for $\lambda=1$ (\autoref{greenEB:fixed_prior}). We also relate it to a full Bayesian posterior interpretation, through a modified PAC-Bayes learning rule \eqref{defn:blue} that corresponds to a profile posterior and has the same limiting behavior as the full Bayesian posterior
    (\autoref{profile_posterior}).
    In particular, all of our lower bounds and inconsistency results hold also for the profile posterior (modified learning rule) and full Bayesian posterior (\autoref{blue:LBresults}). We do not know whether our upper bounds, and consequently consistency and the precise description of the worst case limiting error, extend also to these two learning rules, and leave this as an interesting open question.    
    \end{itemize}

We thus establish a comprehensive analysis of overfitting in such (PAC-) Bayesian frameworks in agnostic supervised learning, characterizing it and understanding how to correct it. The paper is organized as follows. In \autoref{continuous:main}, we present the main results for the PAC-Bayes learning rule $\Qgreen_{\lambda}$  in equation \eqref{eq:PAC-Bayes-intro}. In \autoref{section:EB}, we interpret this learning rule as an Empirical Bayes procedure. In \autoref{section:Bayesian}, we introduce the modified learning rule and then discuss the relationship of both to the full Bayesian Posterior. In \autoref{blue:LBresults}, we give lower bound results for the modified rule. Lastly, we provide proof sketches of the upper and lower bounds of the PAC-Bayes learning rule in \autoref{continuous:upperlowerbound}. The detailed proofs are relegated to the appendices.

\section{Formal Setup} \label{formal setup}
%%%%%%%%%%%%%%%%%%%%%%%%%%%%%%%%%%%%%%%%%%%%%%%%%%%%%%%%%%%%%%%%%%%%%%%%%%%%%%%%
%%%%%%%%%%%%%%%%%%%%%%%%%%%%%%%%%%%%%%%%%%%%%%%%%%%%%%%%%%%%%%%%%%%%%%%%%%%%%%%%

The framework of the paper is the PAC-Bayesian perspective to supervised binary classification with continuous hypothesis classes. We observe $m$ i.i.d samples $S \sim \source^m$ from a source distribution $\source$ over $(X,Y) \in \mathcal{X}\times \{0,1\}$, where $\mathcal{X}$ is some measurable space and $Y$ is a binary label. Predictors (or hypotheses or classifiers) are maps $h:\mathcal{X}\rightarrow\{0,1\}$. Their performance is evaluated using the population error $L_\source(h)=\mathbb{P}_{(x,y) \sim \source}(h(x) \neq y)$. To reduce clutter, we may sometimes omit the subscript $\source$. The observation-derived proxy for this is the training (or empirical) error $L_S(h) = \frac{1}{m} \sum_{i = 1}^m \mathbbm{1}\{h(x_i) \neq y_i\}$. In PAC-Bayes, instead of handling individual predictors, we consider distributions $Q$ over predictors. The population and empirical errors are then understood as the respective expectations over a predictor drawn from $Q$, i.e., respectively $L_{\source}(Q) = \mathbb{E}_{h \sim Q}[L_{\source}(h)]$ and $L_S(Q) = \mathbb{E}_{h \sim Q}[L_S(h)]$.

Given a (\emph{prior}) distribution $\Pi$ over the space of predictors, a learning rule $\hat Q$ is then a sequence of maps $S \mapsto \hat Q(S)$, from
%\natinote{why from $(\source)^m$ and not from $\mathcal{X}^m$? Are you thinking of this as a measurable space?  I find this a bit confusing since I think of the rule as a function for a set of points, nurelated to some distribution, although I guess it's true it needs to be measurable.} % this was wrong, fixed
$(\mathcal{X}\times \{0,1\})^m$ to (\emph{posterior}) distributions over predictors. We are primarily interested in $\hat Q$ that produce posteriors that are absolutely continuous with respect to $\Pi$ and we measure the ``complexity'' of any produced $Q$ via $\KL(Q\Vert \Pi)$. For a given prior distribution $\Pi$, we consider the PAC-Bayes learning rule:
\begin{align}\label{defn:green}
\Qgreen_{\lambda}(S) = \underset{Q}{\text{arg min }} \Jgreen(Q,S), \text{such that } L_S(Q)\leq 1/2,
\end{align}
where 
\begin{align}\label{Gibbs_green}
    \Jgreen(Q,S) = mH(L_S(Q)) +  \lambda \KL(Q \Vert \Pi).
\end{align}
We require $L_S(Q)\leq 1/2$ so that $H(L_S(Q))$ is monotone in $L_S(Q)$---otherwise we might be tempted to choose a predictor with loss close to 1.  For priors $\Pi$ that are symmetric with respect to negating the predictions (roughly speaking, when $\Pi(h)=\Pi(\neg h)$), this can also be interpreted as post-learning negation where necessary.%\natinote{This can be moved to a footnote or removed}

To quantify inconsistency of a learning rule $\hat Q$, we measure it relative to an $L^*$ achievable by a bounded-complexity posterior $Q^*$. We can think of $Q^*$ as a competitor or oracle. Without loss of generality,
%\natinote{This can seem a bit strange.  For the negative results, this is fine.  But on the positive side, we really want to learn anything with $KL>0$ with sample complexity depending on, or even linear in, the $KL$, which we do.  It is worth to at least mention this.}
since our results do not depend on the level of boundedness, we require that $\KL(Q^*\Vert \Pi) \leq 10$. More precisely, when we say ``for a given $L^\star$'', we imply that the \emph{instance} $(\Pi,\source)$ belongs to:
\begin{align}
    \achieve(L^*) = \left\{~(\Pi , \source) ~:\right. \left.\exists Q^*~:~ L_\source(Q^*) = L^*,\ \KL(Q^*\Vert \Pi) \leq 10~\right\}.
\end{align}
This perspective is particularly attuned to the lower bounds. Our upper bounds, on the other hand, apply regardless and imply (possibly tempered) learning for any bounded $\KL(Q^*\Vert \Pi)$.

We can now precisely define just how inconsistent a learning rule can get when $L^*$ is achievable, through its worst-case limiting performance. We can either bound the error uniformly across instances ($\overline{L}_{\infty}$) and then take the limit or we could take the limit per-instance and then bound it across instances ($\underline{L}_{\infty}$). 

\begin{definition}[Worst-case (uniform) limiting error]
    For any $0< L^* < \frac{1}{2}$, regularization parameter $\lambda$, define the worst-case uniform limiting error of a learning rule $\hat Q_{\lambda}(S)$:
    \[
    \overline{L}_{\infty}(\lambda, L^*) =  \underset{m \rightarrow \infty}{\lim}  \underset{\substack{(\Pi , \source) \in \achieve(L^*)}}{\sup} \mathbb{E}_{S\sim \source^m}\left[ L_\source(\hat Q_{\lambda}(S))\right]
    \]
\end{definition}

\begin{definition}[Worst-case (per-instance) limiting error]
    For any $0< L^* < \frac{1}{2}$, regularization parameter $\lambda$, define the worst-case per-instance limiting error of a learning rule $\hat Q_{\lambda}(S)$:
    \[
    \underline{L}_{\infty}(\lambda, L^*) =  \underset{\substack{(\Pi , \source) \in \achieve(L^*)}}{\sup} \underset{m \rightarrow \infty}{\lim} \mathbb{E}_{S\sim \source^m}\left[ L_\source(\hat Q_{\lambda}(S))\right]
    \]
\end{definition}

\removed{\natinote{I think this is repetitive and likely an editing error}
We formally define the PAC-Bayes learning rule as:
\begin{align} \label{eq:PAC-Bayes}
\QEB_{\lambda}(S) = \underset{Q}{\text{arg min }} \JEB(Q,S), \text{such that } L_S(Q)\leq 1/2,
\end{align}
where the objective function is as in equation \eqref{eq:PAC-Bayes}
\begin{align}
    \JEB(Q,S) = mH(L_S(Q)) +  \lambda \KL(Q \Vert \Pi).
\end{align}
The constraint that $L_S(Q)\leq 1/2$ enforces that $\QEB$ does not perform worse than chance on the training data. It is trivially feasible if the hypothesis class contains the all-zero $h=0$ and all-one $h=1$ predictors. 
}

\section{Unifying MDL and PAC-Bayes Limiting Errors}\label{continuous:main} 
%%%%%%%%%%%%%%%%%%%%%%%%%%%%%%%%%%%%%%%%%%%%%%%%%%%%%%%%%%%%%%%%%%%%%%%%%%%%%%%%
%%%%%%%%%%%%%%%%%%%%%%%%%%%%%%%%%%%%%%%%%%%%%%%%%%%%%%%%%%%%%%%%%%%%%%%%%%%%%%%%

We now unify the limiting errors of the PAC-Bayes learning rule \eqref{eq:PAC-Bayes-intro} and its MDL counterpart, by closely paralleling the results of \citet{ZS}. In particular, the upper bounds follow using similar arguments as in the latter, with subtle changes. For the lower bounds, the key change is that we need to show that the posterior concentrates on the desired subset of hypotheses. For that, we exploit the concavity of the entropy.\footnote{All replication of technical language is with permission. Key differences are highlighted in the proofs in the appendix.}

First, we recall the following function $\ell_{\lambda}$ for $0<\lambda<\infty$, introduced in \citet{ZS}:
%The key result that summarizes what follows is that, for any $0<\lambda<\infty$, the worst-case limiting error is precisely upper bounded by the following function $\ell_{\lambda}$, introduced by \citet{ZS}:
\begin{gather}\label{ell}
\ell_{\lambda}(L^*)=
\begin{cases}
1 - 2^{-\frac{1}{\lambda}H(L^*)},  & \text{for } 0< \lambda \leq 1 \\
U_{\lambda}^{-1}(H(L^*)), & \text{for } \lambda > 1,
\end{cases}\\
\textrm{where: } U_{\lambda}(q) = \lambda \KL\!\left(\tfrac{1}{1+ \left(\frac{\scriptscriptstyle 1-q}{\scriptscriptstyle q}\right)^{\frac{\scriptscriptstyle \lambda}{\scriptscriptstyle \lambda - 1}}}\middle\Vert q\right) +\ H\!\left(\tfrac{1}{1+ \left(\frac{\scriptscriptstyle 1-q}{\scriptscriptstyle q}\right)^{\frac{\scriptscriptstyle \lambda}{\scriptscriptstyle \lambda - 1}}}\right).
\end{gather}

We show that, just as in MDL, this function captures the extent of potential inconsistency for each $L^*$ and each choice of $\lambda$, as we elaborate shortly. The precise statements behind this claim are as follows.

\begin{theorem}[Agnostic Upper Bound]
\label{MDL UB}
(1) For any $0<\lambda\leq 1$,  any source distribution $\source$,  any distribution $Q^*$, any valid prior distribution $\Pi$, and any $m$:
\begin{equation}
    \mathbb{E}_{S \sim \source^m}[L_\source(\Qgreen_{\lambda}(S))]
    \leq 1 - 2^{-\frac{1}{\lambda}H(L_\source(Q^*))}
    + O\left(\frac{\KL(Q^* \Vert \Pi)}{m} + \frac{1}{\lambda}\sqrt{\frac{\log^3 (m)}{m}}\right).
\end{equation}

(2)
For any $\lambda > 1$,  any source distribution $\source$, any distribution $Q^*$, any valid prior distribution $\Pi$, and any $m$:
% \begin{equation}
% \begin{split}
%  \mathbb{E}_{S \sim \source^m}[L_\source(\Qgreen_{\lambda}(S))]
%     \leq U_{\lambda}^{-1}\left(H(L_\source(Q^*))\right)+ O \left(\frac{1}{(1 - 2L_\source(Q^*))^2}\phantom{\sqrt{\frac{\log^3 (m)}{m}}}\right.
%     \cdot
%     \left.\left[\lambda\left(\frac{\KL(Q^* \Vert \Pi) + \log m}{m}\right) + \sqrt{\frac{\log^3 (m)}{m}}~\right]\right).
% \end{split}
% \end{equation}
\begin{equation}
    \mathbb{E}_{S \sim \source^m}[L_\source(\Qgreen_{\lambda}(S))]
    \leq U_{\lambda}^{-1}\left(H(L_\source(Q^*))\right)+ O \left(\frac{1}{(1 - 2L_\source(Q^*))^2} \cdot \left[\lambda\left(\frac{\KL(Q^*\Vert\Pi) + \log m}{m}\right) + \sqrt{\frac{\log^3(m)}{m}}\right]\right).
\end{equation}
Where $O(\cdot)$ only hides an absolute constant, that does not depend on $\source, \Pi$, or anything else.
\end{theorem}

While the above finite-sample bounds show that the limiting error will not exceed $\ell_\lambda(L^*)$, we now show that there are worst-case instances that are arbitrarily close to $\ell_\lambda(L^*)$, thus establishing a matching lower bound.

% To establish the exact worst-case limiting error, we provide matching lower bounds, showing that the limiting error can approach $\ell_\lambda(L^*)$, for any $0<\lambda<\infty$ and $L^*$:
\begin{theorem}[Agnostic Lower Bound]
\label{MDL LB}
For any $0<\lambda<\infty$, any $L^* \in (0,0.5)$ and $L^* \leq L' < \ell_{\lambda}(L^*)$, 
%there exists a prior $\Pi$, a distribution $Q^*$ with $\KL(Q^* \Vert \Pi) \leq 10$ and source distribution $\source$ with $L_\source(Q^*) = L^*$ such that
there exists $(\Pi,\source)\in \achieve(L^*)$ such that
$\mathbb{E}_S \left[L_\source(\Qgreen_{\lambda}(S))\right] \rightarrow L'$ as sample size $m \rightarrow \infty$.   
\end{theorem}

As an immediate consequence of \autoref{MDL UB} (with $L_\source(Q^*)=L^*$) and \autoref{MDL LB}, we deduce that just as in MDL, the worst case limiting error is fully given by $\ell_\lambda(L^*)$. Namely, 
% see that $\ell_\lambda(L^*)$ given in \eqref{ell} exactly and tightly characterizes the worst case limiting error:
for any $0<\lambda<\infty$ and any $L^* \in (0,0.5)$, we have
\[
\underline{L}_{\infty}(\lambda, L^*) = \ell_{\lambda}(L^*). 
\]

Moreover, observe that the finite-sample guarantee of \autoref{MDL UB} has a rate of convergence that depends only on $\KL(Q^*\Vert \Pi)$, $\lambda$, and $L^*$ but not on $\Pi$ and $\source$. (As with MDL, the dependence on $L^*$ is only relevant as $L^*$ approaches $1/2$ and can be otherwise ignored.) This implies that the convergence is ``uniform'' and the order of the limits can be interchanged:

% Furthermore, we know that this convergence is ``uniform'', in the sense that we have a finite-sample guarantee ( \autoref{MDL UB}) with sample complexity (i.e., rate of convergence) that depends only on $\KL(Q^*\Vert \Pi) $, $\lambda$ and\footnote{The dependence on $L^*$ only kicks in when $L^*$ is close to $1/2$.  As long as $L^*$ is bounded away from $1/2$, we can ignore this dependence.} $L^*$ but not on $\Pi$ and $\source$.  Another way to view this is that we get the same guarantee even if we change the order of the limits. This fully formalizes our key result, and is captured by the following corollary:

\begin{corollary}\label{cor:finite}
For any $0<\lambda<\infty$, and any $L^* \in (0,0.5)$, 
\[
\begin{split}
% \ell_{\lambda}(L^*) \leq \underset{\Pi, \source }{\sup } \underset{m \rightarrow \infty}{\lim} \mathbb{E}_{S \sim \source^m}\left[ L_\source(\Qgreen_{\lambda})\right]\leq \underset{m \rightarrow \infty}{\lim} \underset{\Pi, \source }{\sup} \mathbb{E}_{S \sim \source^m}\left[ L_\source(\Qgreen_{\lambda})\right] \leq \ell_{\lambda}(L^*).
\ell_{\lambda}(L^*) \leq \underline{L}_{\infty}(\lambda, L^*) \leq \overline{L}_{\infty}(\lambda, L^*) \leq \ell_{\lambda}(L^*).
\end{split}
\]
and so the inequality is actually an equality.
\end{corollary}

The results thus far characterize the behavior of the error for any {\em fixed} $0<\lambda<\infty$ (i.e.~not varying with $m$). See \Cref{fig1} for an illustration. Two key conclusions from MDL thus extend to PAC-Bayes: if $\lambda\geq 1$, the limiting error is always less than $\frac{1}{2}$ and the rule can be interpreted to experience \emph{tempered overfitting}. If $\lambda<1$, this is only true when $L^*<H^{-1}(\lambda)$, otherwise the limiting error could exceed $\frac{1}{2}$. In fact, we notice that as $\lambda$ approaches $0$ (too little regularization), $\ell_\lambda(L^*)$ approaches $1$ (largest error possible), hinting at \emph{catastrophic overfitting}. Can this fate be avoided if we allow $\lambda_m$ to decay to zero with $m$? The following result shows that we cannot, from a worst-case perspective.

% \Cref{cor:finite} fully describes the worst case limiting error (i.e.~overfitting behavior) of $\Qgreen_{\lambda}$ for any {\em fixed} $0<\lambda<\infty$ (i.e.~not varying with $m$).  This is captured by the \citeauthor{ZS} $\ell_\lambda(L^*)$ function, which is portrayed and discussed in \Cref{fig1}.  The behavior of $\ell_\lambda$ as  $\lambda \rightarrow 0$ indicates catastrophic overfitting in this regime, as more explicitly captured by the following:
\begin{theorem}\label{Q:LAMBDA0}
    For any $\lambda_m \rightarrow 0$ or $\lambda = 0$, any $L^* \in (0,0.5)$, and $L^* \leq L' < 1$, %there exists a prior $\Pi$, a distribution $Q^*$ with $\KL(Q^* \Vert \Pi) \leq 10$ and source distribution $\source$ with $L_\source(Q^*) = L^*$ such that
there exists $(\Pi,\source)\in \achieve(L^*)$ such that $\mathbb{E}_S \left[L_\source(\Qgreen_{\lambda})\right] \rightarrow L'$ as sample size $m \rightarrow \infty$.
\end{theorem}

On the positive side, the second part of \Cref{MDL UB} allows us to understand just how to take $\lambda_m \rightarrow\infty$ (more and more regularization), in order to not only get a tempered convergence, but to get consistency or ``learning'', i.e., reach $L^*$. By applying the mean value theorem on $U_\lambda^{-1}(H(\cdot))$ in \Cref{MDL UB}(2), we get:

\begin{theorem}
\label{lambda infty}  
For any distribution $Q^*$, source distribution $\source$, valid prior $\Pi$, and any $\lambda>1$ and m:
\begin{equation}\label{thm3.4bd}
   \mathbb{E}_{S \sim \source^m}[L_\source(\Qgreen_{\lambda})] 
    \leq L_\source(Q^*)+ O\Bigg(\frac{1}{1 - 2L_\source(Q^*)} \cdot
    \left[\frac{1}{\lambda} + 
    \lambda\left(\frac{\KL(Q^* \Vert \Pi) + \log m}{m}\right) + \sqrt{\frac{\log^3 (m)}{m}}~\right]
    \Bigg),
\end{equation}
where $O(\cdot)$ only hides an absolute constant, that does not depend on $\source, \Pi$ or anything else.
\end{theorem}

Similarly to MDL, to obtain consistency, we need a balanced $\lambda_m\rightarrow\infty$ (just enough regularization) to make both $\frac{1}{\lambda}$ and $\lambda\left(\frac{\KL(Q^* \Vert \Pi) + \log m}{m}\right)$ vanish. The optimal choice is $\lambda_m=\sqrt{\frac{m}{\KL(Q^* \Vert \Pi)+\log m}}$ and that gives us an overall rate of $\tilde{O}(1/\sqrt{m})$. More generally, consistency can be achieved in the following range:

% The optimal setting for $\lambda_m$ in \autoref{lambda infty} is $\lambda_m=\sqrt{\frac{m}{\KL(Q^* \Vert \Pi)+\log m}}$, and with any $\lambda_m \propto \sqrt{m}$ we get consistency with rate $\propto \tilde{O}(1/\sqrt{m})$.  More broadly, to get consistency, we need $\lambda_m\rightarrow\infty$  to ensure that $\frac{1}{\lambda}$ vanishes, but also not too fast such that the term $\lambda\left(\frac{\KL(Q^* \Vert \Pi) + \log m}{m}\right)$ also vanishes. This gives the following corollary:

\begin{corollary}\label{lambda_infty_coro}
    For $1 \ll \lambda_m \ll \frac{m}{\log m}$, and any $Q^*$ with $\KL(Q^* \Vert \Pi) < \infty$ and $L_\source(Q^*)<0.5$, we have $\underset{m \rightarrow \infty}{\lim} \underset{\Pi, \source }{\sup}~ \mathbb{E}_{S \sim \source^m}[L_\source(\Qgreen_{\lambda_m})] \leq L_\source(Q^*)$.
\end{corollary}

% \begin{corollary}\label{lambda_infty_coro}
%     For $1 \ll \lambda_m \ll \frac{m}{\log m}$, and $L^*<0.5$, we have $\underset{m \rightarrow \infty}{\lim} \underset{(\Pi, \source)\in \achieve(L^*)}{\sup}~ \mathbb{E}_{S \sim \source^m}[L_\source(\Qgreen_{\lambda_m})] \leq L^*$
% Mesrob: I avoided doing this alternatie version because $\achieve$ forces the KL to be less than 10, and that's not required here. We just need bounded KL. Maybe we can intoduce $\achieve_B$, and use $B=10$ $for the lower bounds and $B=\infinity$ for the upper bounds. (Maybe for the camera-ready version.)
% \end{corollary}

To complete this picture, we show how $\lambda_m$ growing too fast (too much regularization) is harmful and the PAC-Bayes rule $\Qgreen_{\lambda_m}$ exhibits \emph{catastrophic underfitting}:
%However, when $\lambda_m = \Omega(m)$, the PAC-Bayes rule $\Qgreen_{\lambda_m}$ over-regularizes and might ``underfit'', leading to catastrophic behavior again:
\begin{theorem}\label{LAMBDA>>M}
    For any $\lambda_m = \Omega(m)$ with $\liminf \frac{\lambda_m}{m}> 10$, any $L^* \in [0,0.5)$, and any $L^* \leq L' < \frac{0.5 - 0.1L^*}{0.9}$, and for $q^*_m = \arg \min _{q}\lambda_m\KL(q \Vert 0.1) + mH(qL^* + (1-q)L')$, there exists
    $(\Pi,\source)\in \achieve(L^*)$
    %a prior $\Pi$, a distribution $Q^*$ with $\KL(Q^* \Vert \Pi) \leq 10$ and source distribution $\source$ with $L_\source(Q^*) = L^*$
    such that $\mathbb{E} L_\source(\Qgreen_{\lambda_m}(S)) = q^*_mL^* + (1-q^*_m)L' + o(1)$. In particular, if $\frac{\lambda_m}{m} \rightarrow \infty$, then $q^*_m \rightarrow 0.1$ and $\mathbb{E} L_\source(\Qgreen_{\lambda_m}(S)) \rightarrow 0.1L^* + 0.9L'$ as $m \rightarrow \infty$.
\end{theorem}

In this section, we have unified the limiting error behavior of the more general PAC-Bayes rule \eqref{eq:PAC-Bayes-intro} with that of the MDL rule for discrete priors with point-mass posteriors, studied by \citet{ZS}. To summarize this behavior, assuming $L^*<\frac{1}{2}$, we have catastrophic overfitting when $\lambda_m \to 0$ (limiting error reaches $1$), tempered overfitting for fixed $\lambda\geq 1$ (limiting error remains below $\frac{1}{2}$, tempered overfitting for $0<\lambda<1$ only if $L^*<H^{-1}(\lambda)$, benign behavior when $\lambda_m \to \infty$ but remains $o\left(\frac{m}{\log(m)}\right)$ leading to consistency/learning (limiting error reaches $L^*$, with the best convergence rate achieved by $\lambda_m \approx \sqrt{m}$), and, lastly, catastrophic underfitting if $\lambda_m=\omega(m)$ (limiting error reaches $\frac{1}{2}$).

With this unifying perspective on limiting error, which is not evident in advance, we are motivated in the rest of paper to establish more analogies across these paradigms. In particular, we turn our attention to asking how this PAC-Bayes rule relates to Bayesian prediction, and how the regularization parameter $\lambda$, which we see is crucial for consistency, can be interpreted from a probabilistic perspective.

\begin{figure}[h]
    \centering
    \begin{minipage}{0.5\textwidth}
        \centering
        \includegraphics[width=\linewidth]{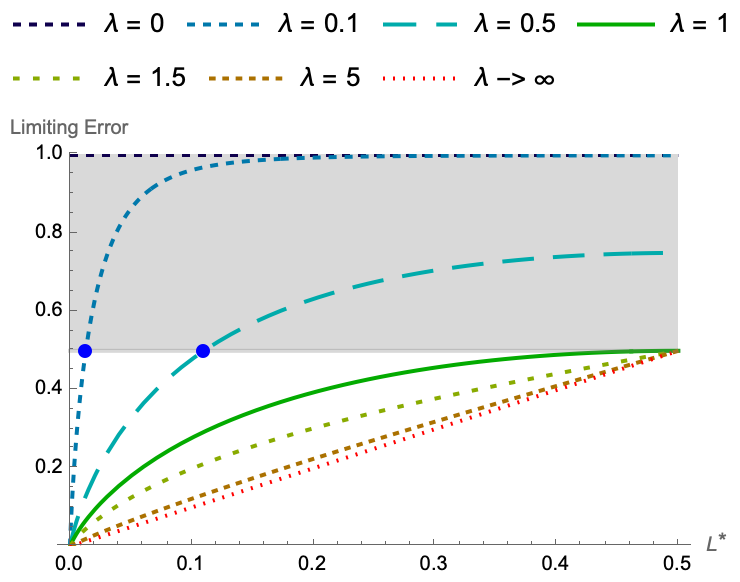}

    \end{minipage}%
    \hfill
    \begin{minipage}{0.45\textwidth}
        \small
        \caption{\small \label{fig1}
  The function $\ell_\lambda(L^*)$ of equation \eqref{ell}. When the PAC-Bayes rule is used with a fixed $\lambda$, the corresponding curve describes the best possible worst-case guarantee on the limiting error, across agnostic noise levels $L^*$. The curves are always above the diagonal (overfitting), but approach it as $\lambda\to \infty$, which is necessary for consistency (learning). For $\lambda\geq 1$, the overfitting is ``tempered'', meaning that the limiting error is less than $\frac{1}{2}$ (better than chance). For $\lambda<1$, this is only the case for $L^*<H^{-1}(\lambda)$, indicated by blue dots. For $\lambda=0$ the overfitting can be catastrophic, with worst case limiting error always $1$. [Reproduced with permission from \citet{ZS}.]}
   \label{fig2}
    \end{minipage}

\end{figure}

\section{PAC-Bayes as Empirical Bayes}\label{section:EB}
%%%%%%%%%%%%%%%%%%%%%%%%%%%%%%%%%%%%%%%%%%%%%%%%%%%%%%%%%%%%%%%%%%%%%%%%%%%%%%%%
%%%%%%%%%%%%%%%%%%%%%%%%%%%%%%%%%%%%%%%%%%%%%%%%%%%%%%%%%%%%%%%%%%%%%%%%%%%%%%%%

It is well-known that the solution of the standard PAC-Bayes posterior has a Gibbs form \citep{mcallester2003pac, catoni2007pac}, where the exponential tilt is given by the errors $L_S(h)$. It turns out that the solution to the PAC-Bayes learning rule \eqref{defn:green}%\natinote{fix reference}
is given by a similar `standard' Gibbs posterior, also featuring $L_S(h)$ in the exponent.

\begin{theorem}\label{thm:EBGibbs}
    Assume $L_S(\Pi)< \frac{1}{2}$. Then, the solution of PAC-Bayes learning rule \eqref{defn:green} is given by:
\begin{align}
    d\QEB_{\lambda}(h) \propto d\Pi(h) 2^{-\left(\frac{m}{\lambda}H'(L_S(\QEB_{\lambda}))\right)\,L_S(h)},
\end{align}
where $L_S(h)$ serves as the energy function of the Gibbs measure and there exists a unique choice of inverse temperature $\frac{m}{\lambda}H'(L_S(\QEB_{\lambda}))$, that is self-consistent.
\end{theorem}
\begin{proof}
This can be shown using a variational argument. A simple way to derive it is to let the Lagrange multiplier be $\mu$, and the Lagrangian $\mathcal{L} =  mH(L_S(Q)) +  \lambda \KL(Q \Vert \Pi) + \mu (\int Q - 1)$. Then by taking the derivative on a single predictor $h_0$, we have that $\frac{\partial \mathcal{L}}{\partial Q(h_0)} = mH'(L_S(Q))L_S(h_0) + \lambda\left(\frac{1}{\ln 2}+ \log \frac{dQ(h_0)}{d\Pi(h_0)}\right) + \mu$. Setting the derivative to zero yields that $d\QEB_{\lambda}(h_0) \propto d\Pi(h_0) 2^{-\frac{m}{\lambda}H'(L_S(\QEB_{\lambda}))L_S(h_0)}$.

Note that the inverse temperature is defined intrinsically by making the solution self-consistent. To see that such a choice exists and is unique, first reparametrize possible choices of posteriors with $\beta=H'(L_S(\hat Q))$. Note that as $\beta$ ranges from $0$ to $\infty$, $L_S(\hat Q_\beta)$ monotonically and continuously decreases from $L_S(\Pi)<\frac{1}{2}$ to $\mathrm{ess}\inf_{h \sim \Pi} L_S(h)\geq 0$. Over the same range of $\beta$, $H^{\prime -1}(\beta) = \frac{1}{1+e^\beta}$ monotonically and continuously decreases from $\frac{1}{2}$ to $0$. These two curves either coincide at $\infty$ or, if not, $H^{\prime -1}(\beta)$ must cross $L_S(\hat Q_\beta)$ at exactly one $\beta$, as desired. This furthermore shows that, under the assumption $L_S(\Pi)<\frac{1}{2}$, we automatically get $L_S(\hat Q)<\frac{1}{2}$ as in \eqref{defn:green}.
\end{proof}

The Gibbs posterior in \autoref{thm:EBGibbs} ``feels'' like  a Bayesian posterior, but can we relate it to one more precisely?  Can we understand the implicitly defined temperature in $\QEB_{\lambda}$ from a Bayesian perspective and describe $\QEB_{\lambda}$ as a Bayesian posterior?

It turns out that $\QEB_{\lambda}$ is an {\em empirical} Bayesian posterior, and thus the learning rule \eqref{defn:green} can be viewed as an empirical Bayes procedure, under a specific source model. Consider the model $y|x = h(x) \oplus \text{Ber}(\eta)$ for some noise level $\eta$, and $h \sim \Pi$ is drawn from the prior distribution $\Pi$. To keep what follows clear, note that this model's joint distribution is of the form
\begin{align}
    \Pi(h)   P(\eta)  P(S|\eta,h)
            \textrm{ with }  P(S|\eta,h) = \eta^{mL_S(h)} (1-\eta)^{m-mL_S(h)}. 
\end{align}
For $\lambda = 1$, $\QEB_1$ is equivalent to the Empirical Bayes procedure consisting of estimating $\eta$ using maximum marginal likelihood and then plugging into the posterior of $h$:
\begin{theorem}\label{greenEB:lambda=1}
    Let $\hat{\eta} = \arg \max_{\eta} P(S|\eta) = \arg \max_{\eta} \mathbb{E}_{h \sim \Pi}P(S|\eta,h)$. Then $\QEB_{1}(h) \equiv P(h|S, \hat{\eta})$.
\end{theorem}

What about for general $\lambda \neq 1$? In this case, $\QEB_{\lambda}$ can once again be viewed as Empirical Bayes with a specific $m$-dependent prior $P_{\lambda}(\eta) = ((1-\eta)^\lambda + \eta^\lambda)^{-\frac{m}{\lambda}}$ on the noise level $\eta$. 
\begin{theorem}\label{greenEB:m_prior}
    Consider the prior $P_{\lambda}(\eta) = ((1-\eta)^\lambda + \eta^\lambda)^{-\frac{m}{\lambda}}$. Let $\hat{\eta} = \arg \max_{\eta} P(S|\eta)P_{\lambda}(\eta) = \arg \max_{\eta} \mathbb{E}_{h \sim \Pi}P(S|\eta,h)P_{\lambda}(\eta)$. Then $\QEB_{\lambda}(h) \equiv P(h|S, \hat{\eta}) \propto \Pi(h) P(S|\hat{\eta}, h) $.
\end{theorem}

Note that when $\lambda = 1$, the prior $P_{\lambda}(\eta) = ((1-\eta)^\lambda + \eta^\lambda)^{-\frac{m}{\lambda}} = 1$, i.e., uniform prior, which reduces to the previous case. Moreover, the following theorem shows that any fixed prior on $\eta$ that does not scale with $m$ behaves the same as uniform prior (i.e., $\lambda=1$).

\begin{theorem}\label{greenEB:fixed_prior}
    For any fixed prior $P_{\lambda}(\eta)$ which does not depend on $m$, $\frac{dP(h|S, \hat{\eta})}{d \QEB_{1}(h)} \rightarrow 1$ uniformly as $m \rightarrow \infty$.
\end{theorem}

We see here a direct probabilistic interpretation of the the PAC-Bayes predictor, which justifies the choice of $\lambda=1$ in order to recover ``standard'' Empirical Bayes.  But the analysis of \autoref{continuous:main} tells us that this choice leads to inconsistency in the worst case, and the only way to avoid such inconsistency (or even get better tempered behavior with $\lambda>1$) is to use a ``prior'' on the noise level that depends on the training set size, and thus significantly departs from Bayesian model selection principles.

\section{Full Bayesian Posterior and Profile Posterior}\label{section:Bayesian}

In \autoref{section:EB}, the Empirical Bayes characterization of $\QEB$ replaces the MAP characterization of $\MDL$ with Bayesian aggregation, and extends it also to $\lambda\neq 1$.  But the use of a point-estimate for the single nuisance parameter $\eta$ still falls short of relating $\QEB$ to a full Bayesian posterior, which would correspond to integrating over $\eta$.  Can we relate $\QEB$ to a Bayesian posterior?  Does the limiting error analysis of \autoref{continuous:main} shed light on the overfitting behavior of fully Bayesian prediction?  

Although we cannot relate $\QEB$ to the full Bayesian posterior, we can show a link by introducing the following modified learning rule:

\begin{align}\label{defn:blue}
\QPL_{\lambda}(S) = \underset{Q}{\text{arg min }} \JPL(Q,S), \text{ such that } L_S(Q)\leq 1/2, \notag
\end{align}
where 
\begin{align}
    \JPL(Q,S) = m\underbrace{\mathbb{E}_{h\sim Q} \left[ H(L_S(h)) \right]}_{\textrm{modified}}+  \lambda \KL(Q \Vert \Pi).
\end{align}

The difference here versus $\QEB$ is the order of expectation and application of the binary entropy functions---we are now applying the entropy function separately to each $L_S(h)$.  This arguably better mimics the average codelength interpretation of MDL. 

The modified $\QPL$ also admits a Gibbs form.  But the exponential tilting function is now proportional to the {\em entropy} of the empirical error\footnote{This can be seen as a special case of a general characterization of \citet{catoni2007pac}, and follows from similar variational analysis as in \autoref{thm:EBGibbs}.}, further deviating from the standard Gibbs form of traditional PAC-Bayes:
\begin{equation}\label{Gibbs_blue}
    d\QPL_{\lambda}(h) \propto d\Pi(h) 2^{-\frac{m}{\lambda}\overbrace{H(L_S(h))}^{\textrm{modified}}},
\end{equation}
with $H(L_S(h))$ as the energy function and $\frac{m}{\lambda}$ as the inverse temperature.  On the positive side, the temperature is now explicit, unlike in \autoref{thm:EBGibbs}.

Does this modified Gibbs posterior also have a Bayesian interpretation? Indeed, it does. In particular, instead of selecting one $\eta$ by maximizing the marginal likelihood, we select a different one for each $h$ by maximizing the full likelihood $\hat{\eta}_h = \arg\max_\eta P(S \mid \eta, h)$. Then $\QPL_{1}(h)$ for $\lambda = 1$ is the posterior of $h$ given $\hat{\eta}_h$:
\begin{theorem}\label{profile_posterior}%\natinote{I changed from quiv to propto, here and in the next Theorem.  Please check}
    Let $\hat{\eta}_h = \arg\max_\eta P(S \mid \eta, h)$. Then $$\QPL_1(h) = P(h \mid S, \hat{\eta}_h) \propto \max_\eta \Pi(h)P(S \mid \eta, h).$$
\end{theorem}
\begin{proof}
    The likelihood $P(S \mid \eta, h) = \eta^{mL_S(h)}(1-\eta)^{m(1 - L_S(h))}$ is maximized at $\hat{\eta}_h = L_S(h)$. Then $P(h \mid S, \hat{\eta}_h) \propto \Pi(h)P(S \mid \hat{\eta}_h, h) = 2^{-mH(L_S(h))}\Pi(h) = \QPL_1(h)$.
\end{proof}
 
This $P(S \mid \hat{\eta}_h, h)$ is known as the ``profile posterior'' (or the method as ``profile likelihood''), and can be interpreted as maximizing out the nuisance parameter in the full likelihood \citep[e.g.,][]{huang2019profile}. We therefore refer to the modified $\QPL_\lambda$ as the Profile Posterior, even for $\lambda \neq 1$.  Indeed, if we use the same $\lambda$-inducing prior as before, $P_{\lambda}(\eta) = ((1-\eta)^\lambda + \eta^\lambda)^{-\frac{m}{\lambda}}$, we can recover $\QPL_{\lambda}$ for general $\lambda$. 
\begin{theorem}\label{PL_lambda_prior}
    Consider the prior $P_{\lambda}(\eta) = ((1-\eta)^\lambda + \eta^\lambda)^{-\frac{m}{\lambda}}$. Let $\hat{\eta}_h = \arg \max_{\eta} P(S \mid \eta,h)P_{\lambda}(\eta)$. Then
    \begin{align*}
    \QPL_{\lambda}(h) \propto \Pi(h)P(S \mid \hat{\eta}_h, h)P_{\lambda}(\hat{\eta}_h)
    \propto \max_\eta \Pi(h)P(S|\eta,h)P_\lambda(\eta).
    \end{align*}
\end{theorem}

Now, given an arbitrary prior $P(\eta)$, consider the Bayesian posterior
\begin{eqnarray}
\QBayes(h)&=& P(h\mid S) = \int_{\eta} P(h\mid S, \eta)P(\eta \mid S)d\eta \nonumber\\
        &\propto& \Pi(h)\int_{\eta}P(S\mid \eta, h)P(\eta)d\eta.    
\end{eqnarray}
When $P(\eta)=P_\lambda(\eta)$, call this posterior $\QBayes_\lambda(h)$. How closely does $\QBayes_\lambda$ relate to $\QPL_{\lambda}$? We first show that for $\lambda=1$, they are asymptotically equivalent. 
\begin{theorem} \label{thm:mod-bayes-equivalence-lambda-1}
$\frac{d\QBayes_1(h)}{d \QPL_{1}(h)} \rightarrow 1$ uniformly as $m \rightarrow \infty$.
\end{theorem}
 
With the same prior $P_{\lambda}(\eta)$, $\QBayes$ is equivalent to $\QPL_{\lambda}$ for general $\lambda$. 
\begin{theorem} \label{thm:mod-bayes-equivalence-any-lambda}
    With the same prior $P_{\lambda}(\eta) = ((1-\eta)^\lambda + \eta^\lambda)^{-\frac{m}{\lambda}}$, $\frac{d\QBayes_\lambda(h)}{d \QPL_{\lambda}(h)} \rightarrow 1$ uniformly as $m \rightarrow \infty$.
\end{theorem}

Similarly, any fixed prior on $\eta$ that does not scale with $m$ behaves the same way as the uniform prior (i.e., $\lambda=1$).
\begin{theorem}\label{thm:mod-bayes-equivalence-fixed-prior}
    With any fixed prior that does not depend on $m$, $\frac{d\QBayes(h)}{d \QPL_{1}(h)} \rightarrow 1$ uniformly as $m \rightarrow \infty$.
\end{theorem}

These results establish that $\QBayes$ and $\QPL_{\lambda}$ are closely related, which implies that they are ``per-instance'' equivalent\footnote{We can prove ``per-instance'' equivalence by bounding the total variation between $\QBayes$ and $\QPL_{\lambda}$.}. In particular, this implies that $\QBayes$ and $\QPL_{\lambda}$ have the same worst-case limiting error. In this way, the overfitting behavior of Bayesian posterior is the same as $\QPL_{\lambda}$. Next, we take steps toward characterizing this limiting behavior and exploring how it relates to that of the PAC-Bayes learning rule.

\section{Limiting Behavior of Profile Posterior}\label{blue:LBresults}

In this section, we give lower bounds for $\QPL$ that match those of $\QEB$. Along with the results on the equivalence of $\QPL$ and $\QBayes$ from the previous section, we can conclude that all three learning rules have the same lower bound on the their worst-case limiting error, governed by the tempering function $\ell_\lambda(L^*)$. In particular, these results mean that claims made about the inconsistency of $\QEB$, namely at $\lambda=1$, for any fixed $\eta$-prior, or at any bounded $\lambda$, all apply to $\QPL$ and $\QBayes$ too.

\paragraph{Open Questions} Beyond the worst case instances, the unmodified empirical Bayes $\QEB$ does not seem to be instance-equivalent to the modified profile posterior $\QPL$ and to the Bayesian $\QBayes$, at least for extreme values of $\lambda$. We are yet unable to provide matching upper bounds and fully characterize the worst case limiting error of these predictors, nor establish consistency when $1\ll\lambda\ll m/\log m$, and leave these as interesting open questions.

We now state the aforementioned lower bound results for the Profile Posterior learning rule $\Qblue_{\lambda_m}$ as in \eqref{defn:blue}, which match those of the PAC-Bayes rule $\QEB_{\lambda_m}$. The proofs are provided in \autoref{appendix_blue}. 

For any $0<\lambda<\infty$, we show that the worst-case limiting error is lower bounded by the same tempering function $\ell_{\lambda}$ as in \eqref{ell}:
\begin{theorem}[Agnostic Lower Bound]\label{B:LBFIXED}
    For any $0<\lambda<\infty$, any $L^* \in (0,0.5)$ and $L^* \leq L' <\ell_{\lambda}(L^*)$, % there exists a prior $\Pi$, a distribution $Q^*$ with $\KL(Q^* \Vert \Pi) \leq 10$ and source distribution $D$ with $L_\source(Q^*) = L^*$ such that 
    there exists $(\Pi,\source)\in \achieve(L^*)$ such that
    $\mathbb{E}_S \left[L_\source(\Qblue_{\lambda})\right] \rightarrow L'$ as sample size $m \rightarrow \infty$. 
\end{theorem}
$\Qblue_{\lambda_m}$ at $\lambda = 0$ or $\lambda_m \rightarrow 0$ exhibits catastrophic overfitting with the limiting error 1:
\begin{theorem}\label{B:LBLAMBDA0}
    For any $\lambda_m \rightarrow 0$ or $\lambda = 0$, any $L^* \in (0,0.5)$, and $L^* \leq L' < 1$, %there exists a prior $\Pi$, a distribution $Q^*$ with $\KL(Q^* \Vert \Pi) \leq 10$ and source distribution $D$ with $L_\source(Q^*) = L^*$ such that 
    there exists $(\Pi,\source)\in \achieve(L^*)$ such that
    $\mathbb{E}_S \left[L_\source(\Qblue_{\lambda})\right] \rightarrow L'$ as sample size $m \rightarrow \infty$.
\end{theorem}

On the other direction, when $\lambda_m = \Omega(m)$, $\Qblue_{\lambda_m}$ leads to catastrophic over-regularizing / underfitting:
\begin{theorem}\label{B:LBLAMBDAINF}
For any $\lambda_m = \Omega(m)$ with $\liminf \frac{\lambda_m}{m}> 10$, any $L^* \in [0,0.5)$, and any $L^* \leq L' < \frac{0.5 - 0.1L^*}{0.9}$, and for $q^*_m =\sigma\left(\text{logit}(0.1) + \frac{m}{\lambda_m}(H(L') - H(L^*))\right)$ where $\sigma(t) = 1/(1+e^{-t})$, %there exists a prior $\Pi$, a distribution $Q^*$ with $\KL(Q^* \Vert \Pi) \leq 10$ and source distribution $D$ with $L_\source(Q^*) = L^*$ such that 
there exists $(\Pi,\source)\in \achieve(L^*)$ such that
$\mathbb{E}_S \left[L_\source(\Qblue_{\lambda_m}(S))\right] = q^*_mL^* + (1-q^*_m)L' + o(1)$. In particular, if $\frac{\lambda_m}{m} \rightarrow \infty$, then $q^*_m \rightarrow 0.1$ and $\mathbb{E}_S \left[L_\source(\Qblue_{\lambda_m}(S))\right] \rightarrow 0.1L^* + 0.9L'$ as $m \rightarrow \infty$.
\end{theorem}
\section{Overview of PAC Bayes Upper and Lower Bounds Proofs}\label{continuous:upperlowerbound}

All detailed proofs are relegated to the Appendices. In this section, in order to convey an idea of the techniques involved, we outline several upper and lower bound proof sketches for the PAC Bayes learning rule. 

\subsection{Upper Bound for $0 < \lambda < \infty$}\label{sec:finite UB proof}
We provide proof sketches for our main upper bound result (\autoref{MDL UB}). We follow closely the methodology of \citet{ZS}, while carefully adapting the discrete prior to the continuous prior case. The following is a sketch of the proof for the PAC-Bayes rule lower bound for $0 < \lambda < \infty$, \autoref{MDL UB}. Details and proofs of other upper bound results for this rule can be found in \autoref{appendix:green_upper_bound}.

Our upper bound results follow from the following generalization guarantee derived from the PAC-Bayes inequality:
\begin{lemma}\label{lemma6.1}
For some constant $C$, any $0 <\lambda < \infty$, with probability $1- \delta$ over $S\sim\source^m$, for any distribution $Q^*$:
\begin{gather}
    \entropicT_{\lambda}(L_\source(\Qgreen_{\lambda}(S))) \leq H(L_\source(Q^*)) + {\lambda}\frac{\log(\frac{m+1}{\delta/2})}{m}+{\lambda}\frac{\KL(Q^* \Vert \Pi)}{m}
    + \sqrt{\frac{2(\log m)^2 \cdot \log \frac{1}{\delta/2}}{m}} \notag\\
\textrm{where:}\quad\quad \entropicT_{\lambda}(q) = \min_{0 \leq p \leq 0.5}{\lambda}\KL(p \Vert q) + H(p) \label{eq:T}
\end{gather}
\end{lemma}

From \Cref{lemma6.1}, it is immediate that as $m \rightarrow \infty$, $L_\source(\Qgreen_{\lambda}(S)) \rightarrow \entropicT^{-1}_{\lambda}(H(L_\source(Q^*)))=\ell_{\lambda}(L_\source(Q^*))$.  Therefore, to derive our finite-sample guarantee results, we analyze applying $\entropicT^{-1}_{\lambda}$ to the right-hand-side in \Cref{lemma6.1}, which involves simplifying the term $\entropicT^{-1}_{\lambda}(H(L_\source(Q^*)))$. The optimum that gives $\entropicT_{\lambda}(q)$ in equation \eqref{eq:T} is different based on the range of $\lambda$. This is the source of the two possible behaviors of $\ell_{\lambda}(L^*)$ in \eqref{ell}.

\begin{proof} {\textbf{of \autoref{MDL UB}} (sketch):}
For $0 < \lambda \leq 1$ and $0 \leq q \leq 1/2$, the optimum is at $p^* = 0$. So in this case, $\entropicT_{\lambda}(q) = - \lambda \log(1-q)$, and the limiting error is $\entropicT_{\lambda}^{-1}(H(L_\source(Q^*))) = 1 - 2^{-\frac{1}{\lambda}H(L_\source(Q^*))}$.

When $1 < \lambda < \infty$, the minimizer $p^* = \frac{1}{1+ (\frac{1-q}{q})^{\frac{\lambda}{\lambda - 1}}}$. So
$\entropicT_{\lambda}(q) =U_\lambda(q)$, and the limiting error is $U_\lambda^{-1}(H(L_\source(Q^*)))$.
To get the finite sample guarantee, we apply $\entropicT^{-1}_{\lambda}$ to both sides of \Cref{lemma6.1}, and then use the mean value theorem on the right hand side.
\end{proof}
\vspace{-10pt}
% \begin{proof} {\bf of \autoref{lambda infty}, $1\ll\lambda$:}

% As $\lambda\rightarrow\infty$, the first term inside the definition of $\entropicT_{\lambda}$ (equation \eqref{eq:T}) dominates, the minimizer is $p^*=q$, and we have $\entropicT_{\lambda}(q)\rightarrow H(q)$.  We would therefore like to apply $H^{-1}$ to both sides of \autoref{lemma6.1} to obtain a bound on $L_\source(\Qgreen_{\lambda}(S))-L_\source(Q^*)$. 
% \end{proof}

% \begin{proof} {\bf of \autoref{lambda_infty_coro},  Consistency when $1 \ll \lambda_m \ll m/\log m$.}

% Since $1 \ll \lambda_m \ll m/\log m$ , as $m \rightarrow \infty$, all the terms inside the big-O notation on the right hand side of equation \eqref{thm3.4bd} in Theorem \ref{lambda infty} vanish, yielding the consistency result.
% \end{proof}

% \section{Lower Bound Constructions and Proof Sketch}\label{continuous:lowerbound}

\subsection{Lower Bound for $0 < \lambda < \infty$}\label{sec:finite LB proof}

Most of our lower bounds have constructive proofs to show that the limiting error cannot improve on the bounds. The following is a sketch of the proof for the PAC-Bayes rule lower bound for $0 < \lambda < \infty$, \autoref{MDL LB}. Details and proofs of other lower bound results for this rule can be found in \autoref{AppendixB}.

\fbox{\parbox{0.97\linewidth}{
\textbf{PAC-Bayes vs. MDL \citep{ZS}.}
The lower bound proof for the PAC-Bayes rule relies on the same lower bound construction as the discrete prior for MDL \citep{ZS}. Specifically, we choose the competitor $Q^* = \delta_{h^*}$ to be the point mass on the ``good'' predictor $h^*$, and prove that the posterior $\Qgreen_{\lambda}$ concentrates on the subset of ``bad'' hypotheses. The extra technicality of the proof, compared with the discrete case, is the exploitation of the concavity of the entropy function.

\quad Note that for the discrete prior $\pi$, we require the `good' predictor $h^*$ to satisfy $\pi(h^*) \geq 0.1$, or above some fixed positive amount. In the continuous case, we require $\KL(Q^*\|\Pi) \leq 10$, or below some fixed bounded. For a point mass $QQ^*=\delta_h$, asking $\KL(\delta_h\|\Pi) = -\log \Pi(h)$ to be below a fixed bound is indeed equivalent to the discrete prior being above some fixed positive amount.
}}
\begin{proof}{\textbf{\autoref{MDL LB}} (sketch)}
For any $0<\lambda<\infty$, any $0<L^*<0.5$, and any $L^* \leq L' < \ell_{\lambda}(L^*)$, we explicitly construct a source distribution (hard learning problem) $\source$ and a prior $\Pi$, and show a distribution $Q^*$ with $\KL(Q^* \Vert \Pi) \leq 10$ and $L_\source(Q^*) = L^*$, such that $\mathbb{E}_S \left[L_\source(\Qgreen_{\lambda}(S))\right] \rightarrow L'$ as the sample size increases ($m \rightarrow \infty$).

We adopt the same construction as in \citep{ZS}, Section 6.1, where each hypothesis extracts a binary coordinate from the feature vector, with a source $\source$ that makes $h_0$ be the ``good'' hypothesis with error $L^\star$ and all other $h_i$ have error $L'>L^*$. We choose a universal prior $\Pi(h_i)=1/(i~\cdot~ \log^2i+10)$. We emphasize that this construction is based on a \emph{discrete} hypothesis space and \emph{discrete} prior, which is permissible. In the PAC-Bayes notation, we thus have that $Q^* = \delta_{h_0}$ is the point mass on $h_0$ with $L^*=L_\source(Q^*)$. $Q^*$ satisfies the needed $\KL(Q^* \Vert \Pi) = -\log \Pi(h_0) = \log 10 < 10$.

% We will construct a distribution $\source$ over infinite binary sequences $x = x[0] x[1]... \in \mathcal{X} = \{0,1\}^\infty$ and binary labels $y\in\{\pm1\}$, and a (discrete) prior over hypothesis $h_i(x)=x[i]$ with\footnote{This is a simple and explicit ``universal'' prior, in the sense that $-\log \Pi(h_i) =\log i+O(\log\log i)$, and it ensures $\Pi(h_0)=0.1$ (we treat $0\cdot \log^2 0 = 0$).} $\Pi(h_i)=1/(i \cdot \log^2 i + 10)$, where each hypothesis is based on one bit of the input (this just allows us to directly specify the joint distribution over the behavior of the hypothesis by specifying the distribution of $x$).  In our constructions $h_0(x)=x[0]$ will always be the ``good'' predictor, $h^*=h_0$, with low population error $L_\source(h_0)=\Pr[x[0]\neq y]=L^*$, while all $h_i$, $i\geq 1$, will be ``bad'', with $L_\source(h_i)=L'>L^*$.

% Given $L^*,L'$, we consider a source distribution $\source$ where $y = \text{Ber}(\frac{1}{2})$, and each bit $x[i]$ is independent conditioned on $y$, with $x[0] = y \oplus \text{Ber}(L^*)$, while $x[i] = y \oplus \text{Ber}(L')$. This ensures $L_\source(h_0)=L^*$ while $L_\source(h_i)=L'$ for $i\geq 1$.

We then prove that, in this construction, as $m\rightarrow\infty$, as long as $L^* \leq L' < \ell_{\lambda}(L^*)$, $\Qgreen_{\lambda}$ will place zero mass on $h_0$ (in the limit), i.e.,~$\Pr_{S\sim \source^m}\left[\Qgreen_\lambda(h_0) > 0\right]\xrightarrow{m\rightarrow\infty} 0$ and thus $L_\source(\Qgreen_\lambda(S))\xrightarrow{p} L'$, as needed. We do this by showing that, in the limit, there is always an alternative posterior $\delta_{h_{\hat{i}}}$ that is a point mass on a single hypothesis $h_{\hat{i}}$, such that if $\hat Q_\lambda$ places any positive mass on $h_0$, it will have a higher (worse) value of the objective $\Jgreen(Q, S)$ \eqref{Gibbs_green} than $\delta_{h_{\hat{i}}}$.
\begin{enumerate}
    \item $\lambda \leq 1$: In this regime, we use the point mass distribution on an interpolating (zero training error) predictor $h_{\hat{i}}$ as an alternative.
    \item $\lambda > 1$: In this regime, we use the empirical error minimizer among the first $k(m)$ bad predictors, with the right choice of the function $k$, as the alternative $h_{\hat{i}}$. \qedhere
\end{enumerate}

\vspace{-12pt}
\end{proof}

\section{Summary}

We have two major contributions in this paper. First, we provided a tight analysis, with matching upper bounds and worst-case lower bounds, on the limiting error of the PAC-Bayes learning rule $\Qgreen_{\lambda}$, based on an arbitrary continuous prior, for any $0<\lambda<\infty$, and generalizes to any $\lambda$. In this sense, we fully extended the analysis of \citet{ZS} to and unified its discrete MDL framework with a continuous PAC-Bayes type rule. We also considered a related learning rule $\Qblue_{\lambda}$ for which we establish the same lower bounds, but left the upper bounds as an open question.

Second, we provided a Bayesian interpretation of the PAC-Bayes predictor $\QEB/\lambda$. Although the Gibbs form of the PAC-Bayes predictor is well known, and the name itself implies a strong relationship with Bayesian analysis, as far as we are aware an explicit link of this nature was not previously made. (Perhaps the closest past attempt at a Bayesian interpretation is the work of \citet{germain2016pac}, which however requires choosing log-likelihood as the loss function within a well-specified model, in contrast to the agnostic zero-one loss setting studied here.) In particular, we showed how  these distinct PAC-Bayes bound minimizers correspond to Empirical Bayes and to the Profile Posterior, and how to interpret their general form with a regularization parameter of $\lambda \neq 1$. 

To further complete the picture, in addition to adding the consistency characterization of Profile Posterior and the full Bayesian posterior, it would also be interesting to relate them explicitly to a more standard PAC-Bayes predictor minimizing $L_S(Q)+\lambda \KL(Q\Vert \Pi)$, without an entropy transformation, which corresponds to a different choice of the Gibbs posterior temperature.

\clearpage

\bibliographystyle{plainnat}
\bibliography{Reference}

\newpage
\appendix

\section{Upper Bounds for the PAC-Bayes Learning Rule}\label{appendix:green_upper_bound}
In this Section, we provide proof sketches for \autoref{MDL UB} and \autoref{lambda infty}, \Cref{cor:finite} and \Cref{lambda_infty_coro}. The proof follows closely the methodology of \citet{ZS}, while carefully adapting the discrete prior to the continuous prior case, the details of which can be found in Appendix A of \citep{ZS}.

We derive our upper bounds from the core generalization guarantee stated below, which relies on the PAC-Bayes bound \citep{mcallester2003simplified}:
\begin{lemma}\label{lemma8}
For some constant $C$, any $0 <\lambda < \infty$, with probability $1- \delta$ over $S\sim\source^m$, for any distribution $Q^*$:
\begin{gather}
    \entropicT_{\lambda}(L_\source(\Qgreen_{\lambda}(S))) \leq H(L_\source(Q^*)) + {\lambda}\frac{\log(\frac{m+1}{\delta/2})}{m}+{\lambda}\frac{\KL(Q^* \Vert \Pi)}{m}
    + \sqrt{\frac{2(\log m)^2 \cdot \log \frac{1}{\delta/2}}{m}} \notag\\
\textrm{where:}\quad\quad \entropicT_{\lambda}(q) = \min_{0 \leq p \leq 0.5}{\lambda}\KL(p \Vert q) + H(p) \label{eq:T2}
\end{gather}
\end{lemma}
\fbox{\parbox{0.99\linewidth}{
\textbf{PAC-Bayes vs. MDL \citep{ZS}.}
The MDL with discrete prior $\pi$ competes with a predictor $h^*$ and pays $|h^*|_\pi=-\log\pi(h^*)$,
whereas with the continuous prior distribution $\Pi$, PAC-Bayes competes with a distribution $Q^*$ and pays $\KL(Q^*\|\Pi)$ (Indeed, for a point mass $Q=\delta_h$, one has $\KL(\delta_h\|\Pi) = -\log \Pi(h)$, recovering the discrete penalty). For notational consistency, note that $T_\lambda$ is the $Q_\lambda$ function defined in the discrete case in \citet{ZS}.
}}
\begin{proof}
We begin with a standard concentration statement that upper-bounds the KL divergence between the empirical error and population error.  This is a special case of the PAC-Bayes bound \citep[][Equation (4)]{mcallester2003simplified}:
%, and is obtained directly by taking a union bound over a binomial tail bound\footnote{More specifically, by applying the binomial tail bound of Theorem 1 in \citet{binomials} to each predictor $h$ in the support of $\pi$, with per-predictor failure probability $\delta_h=\pi(h)\delta/2$, and taking a union bound over all $h$.}
\begin{equation}\label{pacybayes1}
\begin{split}
\Pr_{S\sim\source^m}\left[\forall_Q \KL\left(L_S(Q) \middle\Vert  L_\source(Q) \right) \leq \frac{\KL(Q \Vert \Pi) + \log(\frac{m+1}{\delta/2})}{m} \right]
\geq 1-\delta/2.
\end{split}
\end{equation}
Specializing \eqref{pacybayes1} to $Q = \Qgreen_{\lambda}$, multiplying both sides by $\lambda$, and adding $H(L_S(\Qgreen_{\lambda}))$ to both sides, we obtain with probability $\geq 1-\delta/2$,
\begin{align}
{\lambda}\KL(L_S(\Qgreen_{\lambda}) \Vert L_\source(\Qgreen_{\lambda})) + H(L_S(\Qgreen_{\lambda}))
&\leq H(L_S(\Qgreen_{\lambda})) +{\lambda}\frac{\KL(\Qgreen_{\lambda} \Vert \Pi)}{m} + {\lambda}\frac{\log(\frac{m+1}{\delta/2})}{m} \label{eqn8lhs}\\ 
&\leq H(L_S(Q^*)) + {\lambda}\frac{\KL(Q^* \Vert \Pi)}{m} + {\lambda}\frac{\log(\frac{m+1}{\delta/2})}{m},\label{ineqn10}\\
\intertext{and with probability $\geq 1-\delta$:}
&\leq  H(L_\source(Q^*))+{\lambda}\frac{\KL(Q^* \Vert \Pi)}{m} + {\lambda}\frac{\log(\frac{m+1}{\delta/2})}{m}+ \sqrt{\frac{2(\log m)^2 \cdot \log \frac{1}{\delta/2}}{m}},\label{ineqn11}
\end{align}
where the second inequality \eqref{ineqn10} follows from the definition of $\Qgreen_{\lambda}$. In the third inequality \eqref{ineqn11} we use McDiarmid's inequality to bound the difference (with another failure probability of $\delta/2$) between the entropy of the empirical and population loss of the fixed distribution $Q^*$.

Our goal is to get an upper bound on the population error $L_\source(\Qgreen_\lambda(S))$.  The difficulty is that the left-hand-side in \eqref{eqn8lhs} also involves the unknown empirical error $L_S(\Qgreen_\lambda(S))$, for which we only know by definition $L_S(\Qgreen_\lambda(S))\leq1/2$.  Instead, we'll replace this empirical error with $p=L_S(\Qgreen_\lambda(S))$ and minimize \eqref{eqn8lhs} w.r.t~$p$, as in $\entropicT_{\lambda}(L_\source(\Qgreen_{\lambda}(S))) = \min_{0 \leq p \leq 0.5} \lambda\KL(p \Vert L_\source(\Qgreen_{\lambda}(S))) + H(p)$.  From the definition of this $\entropicT_{\lambda}(q)$, we therefore have that $\entropicT_{\lambda}(L_\source(\Qgreen_{\lambda}(S)))$ is upper bounded by \eqref{eqn8lhs}, from which the Lemma follows.
\end{proof}

\Cref{lemma8} already implies the limiting error $L_\source(\Qgreen_{\lambda}(S)) \rightarrow \entropicT^{-1}_{\lambda}(H(L_\source(Q^*)))=\ell_{\lambda}(L_\source(Q^*))$, as $m \rightarrow \infty$. So to get finite sample guarantees stated in the main theorems, we need to further simplify the expression $\entropicT^{-1}_{\lambda}(H(L_\source(Q^*)))$, and also analyze applying $\entropicT^{-1}_{\lambda}$ to the right-hand-side in \Cref{lemma8}.

\begin{proof} {\bf of \autoref{MDL UB} part (1), $0<\lambda\leq 1$:}

For $0 < \lambda \leq 1$ and $0 \leq q \leq 1/2$,  the function ${\lambda}\KL(p \Vert q) + H(p)$ is monotonically increasing in $p$, and thus the minimizer occurs at $p^* = 0$. Consequently, $\entropicT_{\lambda}(q) = - \lambda \log(1-q)$, and by \Cref{lemma8}, the limiting error is $\entropicT_{\lambda}^{-1}(H(L_\source(Q^*))) = 1 - 2^{-\frac{1}{\lambda}H(L_\source(Q^*))}$. To convert this into a finite sample guarantee, we apply the inequality $ 1- 2^{-\alpha-A} \leq 1- 2^{-\alpha} + A$ (for $A,\alpha>0$) \citep[adapted from Lemma A.4 in][]{MS}.
\end{proof}
\vspace{-10pt}
\begin{proof} {\bf of \autoref{MDL UB} part (2), $\lambda > 1$:}

When $1 < \lambda < \infty$, we differentiate $\lambda \cdot \KL(p \Vert q) + H(p)$ w.r.t. $p$, set the derivative to zero, and obtain the minimizer $p^* = \frac{1}{1+ (\frac{1-q}{q})^{\frac{\lambda}{\lambda - 1}}}$. Substituting this value back into the objective in we have 
$\entropicT_{\lambda}(q) = \lambda \cdot \KL\left(\frac{1}{1+ (\frac{1-q}{q})^{\frac{\lambda}{\lambda - 1}}}\middle\Vert q\right) + H\left(\frac{1}{1+ (\frac{1-q}{q})^{\frac{\lambda}{\lambda - 1}}}\right)=U_\lambda(q)$, and the limiting error is $U_\lambda^{-1}(H(L_\source(Q^*)))$.
We apply $U^{-1}_{\lambda}$ to both sides of \Cref{lemma8}, and then invoke the mean value theorem on the right hand side to get the bound with uniform rate. We apply the mean value theorem and bound the derivative of $U_{\lambda}^{-1}$ uniformly in terms of $L_\source(Q^*)$, and this gives a pre-factor of $1/(1-2L(Q^*))$.  
\end{proof}
\vspace{-10pt}
\begin{proof} {\bf of \autoref{lambda infty}, $\lambda \gg 1$:}

As $\lambda\rightarrow\infty$, the $\KL$ term in the definition of $\entropicT_{\lambda}$ (equation \eqref{eq:T2}) dominates, and the minimizer is $p^*=q$, and hence $\entropicT_{\lambda}(q)\rightarrow H(q)$.  This motivates applying $H^{-1}$ to both sides of \Cref{lemma8} to obtain a bound on $L_\source(\Qgreen_{\lambda}(S))-L_\source(Q^*)$.  For finite  (but large) $\lambda$, we rely on the following Lemma, which gives a quantification of how close $U_{\lambda}(q)$ is to the entropy function $H(q)$:
\begin{lemma}\label{lemma9} For any $\lambda >1$ and any $0 \leq q \leq \frac{1}{2}$, $H(q) < U_{\lambda}(q)+ \lambda/(\lambda-1)^2$.
\end{lemma}
By \Cref{lemma8} and \Cref{lemma9}, it implies that with probability at least $ 1 - \delta$, 
\begin{equation}\label{thm3.4H}
\begin{split}
    H(L_\source(\Qgreen_{\lambda}(S))) &\leq H(L_\source(Q^*)) + \sqrt{\frac{2(\log m)^2 \cdot \log \frac{1}{\delta/2}}{m}}
    + {\lambda}\frac{\log(\frac{m+1}{\delta/2})}{m}+{\lambda}\frac{\KL(Q^* \Vert \Pi)}{m} + \frac{\lambda}{(\lambda-1)^2}.
\end{split}
\end{equation}
Applying $H^{-1}$ to both sides of \eqref{thm3.4H}, we then invoke the mean value theorem and bound the derivative of $H^{-1}$ by $\frac{\frac{1}{2} - L_\source(Q^*)}{1 - H(L_\source(Q^*))} \leq \frac{\ln 2}{1 - 2L_\source(Q^*)}$. This gives the stated theorem. \end{proof}

\begin{proof} {\bf of \Cref{lambda_infty_coro},  Consistency when $1 \ll \lambda_m \ll m/\log m$.}

Under the condition $1 \ll \lambda_m \ll m/\log m$ , every term inside the big-O notation on the right hand side of equation \eqref{thm3.4H} in \autoref{lambda infty} vanishes as $m \rightarrow \infty$. This implies the stated consistency conclusion.
\end{proof}
\section{Lower Bounds for the PAC-Bayes Learning Rule}\label{AppendixB}
In this Section, we provide the detailed lower bound proofs for \autoref{MDL LB}, \autoref{Q:LAMBDA0} and \autoref{LAMBDA>>M}.

\subsection{Lower Bound for $0 < \lambda < \infty$ (proof of \autoref{MDL LB})} \label{subsection:B.1}
\restate{Theorem}{MDL LB} 
\begin{restatethm}[Agnostic Lower Bound]
    For any $0<\lambda<\infty$, any $L^* \in (0,0.5)$ and $L^* \leq L' < \ell_{\lambda}(L^*)$, there exists a prior $\Pi$, a distribution $Q^*$ with $\KL(Q^* \Vert \Pi) \leq 10$ and source distribution $\source$ with $L_\source(Q^*) = L^*$ such that $\mathbb{E}_S \left[L_\source(\Qgreen_{\lambda}(S))\right] \rightarrow L'$ as sample size $m \rightarrow \infty$. 
\end{restatethm}
\fbox{\parbox{0.99\linewidth}{
\textbf{PAC-Bayes vs. MDL \citep{ZS}.}
The lower bound proof for the continuous case relies on the same lower bound construction in the discrete case \citep{ZS}. Specifically, we choose the competitor $Q^* = \delta_{h^*}$ to be the point mass on the `good' predictor $h^*$, and prove that the posterior $\Qgreen_{\lambda}$ concentrates on the subset of `bad' hypotheses. The extra technicality of the proof, compared with the discrete case, is the exploitation of the concavity of the entropy function.

Note that for the discrete prior $\pi$, we require the `good' predictor $h^*$ to satisfy $\pi(h^*) \geq 0.1$,
whereas in the continuous case, we require $\KL(Q^*\|\Pi) \leq 10$. (Indeed, for a point mass $Q=\delta_h$, one has $\KL(\delta_h\|\Pi) = -\log \Pi(h)$, recovering the discrete penalty).
}}

We use the same construction $\source$ as in \citep{ZS}, Section 6.1 over infinite binary sequences $x = x[0] x[1]... \in \mathcal{X} = \{0,1\}^\infty$ and binary labels $y\in\{\pm1\}$, where each hypothesis selects a binary coordinate from $y$ as $h_i(x)=x[i]$. We use a (discrete) prior over hypotheses with $\Pi(h_i)=1/(i \cdot \log^2 i + 10)$. Given $L^*,L'$, define $\source$ such that $y = \text{Ber}(\frac{1}{2})$, and $x[0] = y \oplus \text{Ber}(L^*)$, while $x[i] = y \oplus \text{Ber}(L')$, where each bit $x[i]$ is independent conditioned on $y$. For this $\source$, we denote $h_0(x)=x[0]$ to be the ``good'' predictor with low population error $L_\source(h_0)=L^*$, whereas all $h_i$ for $i\geq 1$ are ``bad'' predictors with high population error $L_\source(h_i)=L'>L^*$.

% We want to note that this construction is based on a \emph{discrete} hypothesis space and \emph{discrete} prior, which is permissible. 
Consider $Q^* = \delta_{h_0}$ be the point mass on $h_0$, and then $L_\source(Q^*) = L^*$, and $\KL(Q^* \Vert \Pi) = -\log \Pi(h_0) = \log 10 < 10$.

We show that as $m\rightarrow\infty$, as long as $L^* \leq L' < \ell_{\lambda}(L^*)$, $\Qgreen_{\lambda}$ will place zero mass on $h_0$, i.e.~$\Pr_{S\sim \source^m}\left[\Qgreen_\lambda(h_0) > 0\right]\xrightarrow{m\rightarrow\infty} 0$ and thus $L_\source(\Qgreen_\lambda(S))\xrightarrow{p} L'$.
We discuss the proof for $\lambda \leq 1$ and $\lambda > 1$ separately below.
\vspace{-10pt} 
\begin{enumerate}
    \item $\lambda \leq 1$: Choose $k(m) = \frac{2\sqrt{m}}{(1-L')^m}$, then (with probability approaching one), there exists some ``bad'' classifier $h_{\hat{i}}$ with $1\leq \hat{i} \leq k(m)$ such that $L_S(h_{\hat{i}}) = 0$ (proved in B.1.1 in \citet{ZS}). Consider the point mass on the predictor $h_{\hat{i}}$ as $\delta_{h_{\hat{i}}}$. Then we have $\Jgreen(\delta_{h_{\hat{i}}}, S) = \lambda(-\log \Pi(h_{\hat{i}}))$. Hence, the objective of the minimizer $\Qgreen_{\lambda}$ should be upper bounded by the objective of $\delta_{h_{\hat{i}}}$, i.e., $\Jgreen(\Qgreen_{\lambda}, S) \leq \lambda(-\log \Pi(h_{\hat{i}}))$ as $m \rightarrow \infty$.

    Suppose $\Qgreen_\lambda(h_0) = q_0 > 0$ and $\Qgreen_\lambda(h_{\hat{i}}) =  1- q_0$ for the sake of contradiction. For the ``good'' predictor $h_0$ we have that $L_S(h_0)\xrightarrow{p} L_\source(h_0)=L^*$.  By the definition of $\Pi$ and that $h_{\hat{i}} \in \mathcal{H}_{k(m)}$, we have $-\log \Pi(h_{\hat{i}}) \leq m\log \frac{1}{1-L'} + C\log m$, for some constant $C>0$. So
    
    \begin{align}
        &\Jgreen(\Qgreen_{\lambda}, S) = mH(L_S(\Qgreen_{\lambda})) + \lambda \KL(\Qgreen_{\lambda} \Vert \Pi)
        = mH(q_0L_S(h_0))
         + \lambda \left(q_0\log \frac{q_0}{\Pi(h_0)}+ (1 - q_0)\log \frac{1 - q_0}{\Pi(h_{\hat{i}})}\right)\\
        &\geq mq_0H(L_S(h_0)) + \lambda(1 - q_0)\log \frac{1}{\Pi(h_{\hat{i}})} + O(1) \xrightarrow{p} mq_0H(L^*) + \lambda(1 - q_0)\log \frac{1}{\Pi(h_{\hat{i}})} + O(1)
        \label{Hconcave}\\
        &> -mq_0 \lambda \log (1-L') + \lambda(1 - q_0)\log \frac{1}{\Pi(h_{\hat{i}})} + \Omega(m)\label{ineq28}
        \geq q_0 \lambda \log \frac{1}{\Pi(h_{\hat{i}})} + \lambda(1 - q_0)\log \frac{1}{\Pi(h_{\hat{i}})}+ \Omega(m)\\
        &= \Jgreen(\delta_{h_{\hat{i}}}, S) + \Omega(m),
    \end{align}
    where in the first inequality of \eqref{Hconcave} we used $H(tx) \geq tH(x), \forall t \in (0,1)$ since $H$ is concave and $H(0) = 0$. And in the first inequality of \eqref{ineq28} we used $L' < 1 - 2^{-H(L^*)/{\lambda}}$, and the asymptotic notation is w.r.t.~$m\rightarrow\infty$. This leads to a contradiction. Thus, as long as $L' < 1 - 2^{-H(L^*)/{\lambda}}$, as $m \rightarrow \infty$, the mass on $h_0$ should go to zero with probability one, which implies the bound for the expected risk:
    as $m \rightarrow \infty$, $\mathbb{E}_S \left[L_\source(\Qgreen_{\lambda}(S))\right] \rightarrow L'$.  This completes the proof for $ 0< \lambda \leq 1$.

    \item $\lambda > 1$: Choose $k(m) = 2^{m\KL(\hat{L}\Vert L')}$ where $\hat{L} = \frac{1}{1+ (\frac{1-L'}{L'})^\frac{\lambda}{\lambda - 1}}$. Consider the empirical error minimizer $h_{\hat{i}}$ among the first $k(m)$ bad predictors, i.e.~such that $L_S(h_{\hat{i}}) = \min_{i=1\ldots k(m)} L_S(h_i)$.  This is the minimum of $k(m)$ independent (scaled) binomials $\text{Bin}(m,L')$ so concentrates s.t.~$\KL( L_S(h_{\hat{i}}) \Vert L') \xrightarrow{p} \frac{\log k(m)}{m}=\KL(\hat{L}\Vert L')$(by Theorem 1 in \citet{binomials}), and hence $L_S(h_{\hat{i}}) \xrightarrow{p} \hat{L}$ (proved in B.1.2 in \citet{ZS}). Consider the point mass on the predictor $h_{\hat{i}}$ as $\delta_{h_{\hat{i}}}$. Then we have $\Jgreen(\delta_{h_{\hat{i}}}, S) = mH(L_S(h_{\hat{i}})) +  \lambda(-\log \Pi(h_{\hat{i}}))$. Hence, the objective of the minimizer $\Qgreen_{\lambda}$ should be upper bounded by the objective of $\delta_{h_{\hat{i}}}$, i.e., $\Jgreen(\Qgreen_{\lambda}, S) \leq mH(L_S(h_{\hat{i}})) + \lambda(-\log \Pi(h_{\hat{i}}))$ as $m \rightarrow \infty$.
    
    Suppose $\Qgreen_\lambda(h_0) = q_0 > 0$ and $\Qgreen_\lambda(h_{\hat{i}}) =  1- q_0$ for the sake of contradiction.  By the definition of $\Pi$ and that $h_{\hat{i}} \in \mathcal{H}_{k(m)}$, we have $-\log \Pi(h_{\hat{i}}) \leq m\KL(\hat{L}\Vert L') + C'\log m$, for some $C' > 0$. Then 
    \begin{align}
        &\Jgreen(\Qgreen_{\lambda}, S) = mH(L_S(\Qgreen_{\lambda})) + \lambda \KL(\Qgreen_{\lambda} \Vert \Pi)\\
        &\geq m\left(q_0H(L_S(h_0)) + (1-q_0)H(L_S(h_{\hat{i}}))\right)
        + \lambda(1 - q_0)\log \frac{1}{\Pi(h_{\hat{i}})} + O(1) \label{Hconcave:lambda>1}\\
        % &\xrightarrow{p} mH(q_0L^* + (1-q_0)L_S(h_{\hat{i}}))
        % + \lambda(1 - q_0)\log \frac{1}{\Pi(h_{\hat{i}})} + O(1)\\
        &\xrightarrow{p} mq_0H(L^*) + (1-q_0)\Jgreen(\delta_{h_{\hat{i}}}, S) + O(1)\\
        &> mq_0U_{\lambda}(L') + (1-q_0)\Jgreen(\delta_{h_{\hat{i}}}, S)+ \Omega(m) \label{ineq34}\\
        &= mq_0\left(\lambda \KL(\hat{L}\Vert L')+H(\hat{L})\right)
        + (1-q_0)\Jgreen(\delta_{h_{\hat{i}}}, S)+ \Omega(m)\\
        &\geq q_0\left(\lambda (-\log \Pi(h_{\hat{i}}))+mH(L_S(h_{\hat{i}}))\right)
        + (1-q_0)\Jgreen(\delta_{h_{\hat{i}}}, S)+ \Omega(m)\\       &>q_0\Jgreen(\delta_{h_{\hat{i}}},S) + (1-q_0)\Jgreen(\delta_{h_{\hat{i}}}, S)+ \Omega(m)
        = \Jgreen(\delta_{h_{\hat{i}}},S)+ \Omega(m),
    \end{align}
where in \eqref{Hconcave:lambda>1} we used the concavity of $H$, and in \eqref{ineq34} we relied on  $L' < U_{\lambda}^{-1} ( H(L^*) )$ and the definition of $U_\lambda$, and the asymptotic notation is w.r.t.~$m\rightarrow\infty$. This leads to a contradiction. Hence, as long as  $L' < U_{\lambda}^{-1} ( H(L^*) )$, as $m \rightarrow \infty$, the mass on $h_0$ will go to zero with probability one. 
% It is important to note that in the definition of $\MDL_{\lambda}$, we also require the selected hypothesis $h$ to satisfy $L_S(h) \leq \frac{1}{2}$ (as in equation \eqref{defn:MDL}). And we just showed that with probability one $L_S(h_{\hat{i}}) \rightarrow \hat{L} < L' < U_{\lambda}^{-1} ( H(L^*) ) < \frac{1}{2}$, so $h_{\hat{i}}$ satisfies the condition and has a lower $\MDL$ objective than $h_0$.
    
This implies the bound for the expected risk:
as $m \rightarrow \infty$, $\mathbb{E}_S \left[L_\source(\Qgreen_{\lambda}(S))\right] \rightarrow L'$. This completes the proof for $ 1< \lambda < \infty$. This completes the proof for \autoref{MDL LB}.
\end{enumerate}

% Proof of \autoref{Q:LAMBDA0}.

\subsection{Proof of \autoref{Q:LAMBDA0}}
\restate{Theorem}{Q:LAMBDA0} 
\begin{restatethm}
 For any $\lambda_m \rightarrow 0$ or $\lambda = 0$, any $L^* \in (0,0.5)$, and $L^* \leq L' < 1$, there exists a prior $\Pi$, a distribution $Q^*$ with $\KL(Q^* \Vert \Pi) \leq 10$ and source distribution $\source$ with $L_\source(Q^*) = L^*$ such that $\mathbb{E}_S \left[L_\source(\Qgreen_{\lambda})\right] \rightarrow L'$ as sample size $m \rightarrow \infty$.
\end{restatethm}
\fbox{\parbox{0.99\linewidth}{
\textbf{PAC-Bayes vs. MDL \citep{ZS}.}
We choose the competitor $Q^* = \delta_{h^*}$ to be the point mass on the `good' predictor $h^*$. The learning rule minimizes the empirical error $L_S(Q)$ in this case and we prove that the posterior $\Qgreen_{\lambda}$ concentrates on the subset of `bad' hypotheses. We exploit the concavity of the entropy function for the proof.
}}
\begin{proof}
We use the same source distribution and prior described in \autoref{subsection:B.1}. $Q^* = \delta_{h_0}$ is the point mass on $h_0$, and $L_\source(Q^*) = L^*$, and $\KL(Q^* \Vert \Pi) = -\log \Pi(h_0) = \log 10 < 10$.

For $\lambda = 0$, then the PAC-Bayes learning rule simply minimizes $L_S(Q)$. There exists a.s.~some $\hat{i}$ with $L_S(h_{\hat{i}})=0$, whereas $L_S(h_0)\xrightarrow{p}L^*>0$. Consider $Q'$ being a point mass on $h_{\hat{i}}$, then we have $\Jgreen(Q') = mH(L_S(Q')) +\lambda \KL(Q'\Vert \Pi) = 0$. So by definition of $\Qgreen_{\lambda}$, we must have $\Jgreen(\Qgreen_{\lambda}) = mH(L_S(\Qgreen_{\lambda})) +\lambda \KL(\Qgreen_{\lambda}\Vert \Pi) = mH(L_S(\Qgreen_{\lambda})) \leq 0$. Any $Q$ with a positive mass on $h_0$, i.e., $Q(h_0) > 0$ must have $\Jgreen(Q) = mH(L_S(Q)) > mH(L_S(h_0)Q(h_0)) \geq  mQ(h_0)H(L_S(h_0))\xrightarrow{p} mQ(h_0)H(L^*) > 0$, which leads to a contradiction. Hence, as $m \rightarrow \infty$, $\Qgreen_{\lambda}(h_0)$ should go to zero with probability approaching one, which implies the bound for the expected risk: as $m \rightarrow \infty$, $\mathbb{E}_S \left[L_\source(\Qgreen_{\lambda}(S))\right] \rightarrow L'$.  

For $\lambda_m \rightarrow 0$ as $m \rightarrow \infty$, denote $\hat{i}$ to be the smallest index $\hat{i}\geq 1$ such that $L_S(h_{\hat{i}}) = 0$. We have shown that $\hat{i} \leq\frac{m+1}{(1-L')^m}$ with probability approaching one. Consider $Q'$ being a point mass on $h_{\hat{i}}$, then we have $\Jgreen(Q') = mH(L_S(Q')) +\lambda_m \KL(Q'\Vert \Pi) \leq\lambda_m \log\frac{m+1}{(1-L')^m}=o(m)$ as $\lambda_m \rightarrow 0$. So by definition of $\Qgreen_{\lambda}$, we must have $\Jgreen(\Qgreen_{\lambda}) = mH(L_S(\Qgreen_{\lambda})) +\lambda_m \KL(\Qgreen_{\lambda}\Vert \Pi) \leq o(m)$ as $m$ increases. Any $Q$ with a positive mass on $h_0$, i.e., $Q(h_0) > 0$ must have $\Jgreen(Q) =mH(L_S(Q)) + \lambda_m \KL(Q\Vert \Pi) > mH(L_S(h_0)Q(h_0))\geq  mQ(h_0)H(L_S(h_0)) \xrightarrow{p} mQ(h_0)H(L^*) = \Omega(m)$, which leads to a contradiction. Hence, with probability approaching one, $\Qgreen_{\lambda}(h_0)$ should go to zero, and so $\mathbb{E}_S \left[L_\source(\Qgreen_{\lambda}(S))\right] \rightarrow L'$.   
\end{proof}
\subsection{Proof of \autoref{LAMBDA>>M}}
\restate{Theorem}{LAMBDA>>M} 
\begin{restatethm}
For any $\lambda_m = \Omega(m)$ with $\liminf \frac{\lambda_m}{m}> 10$, any $L^* \in [0,0.5)$, and any $L^* \leq L' < \frac{0.5 - 0.1L^*}{0.9}$, and for $q^*_m = \arg \min _{q}\lambda_m\KL(q \Vert 0.1) + mH(qL^* + (1-q)L')$, there exists a prior $\Pi$, a distribution $Q^*$ with $\KL(Q^* \Vert \Pi) \leq 10$ and source distribution $\source$ with $L_\source(Q^*) = L^*$ such that $\mathbb{E} L_\source(\Qgreen_{\lambda_m}(S)) = q^*_mL^* + (1-q^*_m)L' + o(1)$. In particular, if $\frac{\lambda_m}{m} \rightarrow \infty$, then $q^*_m \rightarrow 0.1$ and $\mathbb{E} L_\source(\Qgreen_{\lambda_m}(S)) \rightarrow 0.1L^* + 0.9L'$ as $m \rightarrow \infty$.
\end{restatethm}

\fbox{\parbox{0.99\linewidth}{
\textbf{PAC-Bayes vs. MDL \citep{ZS}.}
We choose the competitor $Q^* = \delta_{h^*}$ to be the point mass on the `good' predictor $h^*$. Different from the previous cases where we prove the posterior mass on $h_0$ vanishes, in the case $\lambda_m = \Omega(m)$, in particular, if $\lambda_m \rightarrow \infty$, the weight which the posterior $\Qgreen_{\lambda_m}$ puts on each predictor converges to its prior, i.e, we ignore the data and just use the prior at the limit.
}}

\begin{proof}
We use the source distribution described in \autoref{subsection:B.1} with only two predictors $\{h_0, h_1\}$ with the prior $\Pi(h_0)=0.1$ and $\Pi(h_1)= 0.9$.

$Q^* = \delta_{h_0}$ is the point mass on $h_0$, and $L_\source(Q^*) = L^*$, and $\KL(Q^* \Vert \Pi) = -\log \Pi(h_0) = \log 10 < 10$.

Let $\Qgreen_{\lambda_m}$ be the output distribution, and let $\Qgreen_{\lambda_m}(h_0) = q$, $\Qgreen_{\lambda_m}(h_1) = 1-q$. 
\begin{equation}
\begin{split}
    \hat{J}_{\lambda_m}(q, S)=\lambda_m\KL(q \Vert 0.1) + mH(L_S(\Qgreen_{\lambda_m})) = \lambda_m\KL(q \Vert 0.1) + mH(\hat{L}(q)),
\end{split}
\end{equation}
where $\hat{L}(q) = qL_S(h_0) + (1-q)L_S(h_1).$ Denote the minimizer of $\hat{J}_{\lambda_m}(q, S)$ as $\qhat = \arg \min_{q \in (0,1)} \hat{J}_{\lambda_m}(q, S)$.

Let $J_{\lambda_m}(q) = \lambda_m\KL(q \Vert 0.1) + mH(L(q))$, where $L(q) =qL^* + (1-q)L'$. Denote the minimizer $\qstar = \arg\min_{q \in (0,1)} J_{\lambda_m}(q)$, then $\qstar$ satisfies
\begin{align}
    0 = J_{\lambda_m}'(\qstar) = \lambda_m\left(\log 
    \frac{\qstar}{1-\qstar} - \log \frac{0.1}{0.9}\right) + m(L^* - L')\log \frac{1 - L(\qstar)}{L(\qstar)},
\end{align}
or equivalently,
\begin{equation}
\begin{split}
    \text{logit}(\qstar) = \text{logit}(0.1) + \frac{m}{\lambda_m}(L' - L^*)
    \log \frac{1 - L(\qstar)}{L(\qstar)}.
    % &> \log \frac{0.1}{1-0.1}.
\end{split}
\end{equation}
Note that $J_{\lambda_m}'$ is continuous on $(0, 1)$, and differentiating again,
\begin{align}\label{second_deriv}
    J_{\lambda_m}''(q) = \frac{\lambda_m}{q(1-q)} - \frac{m(L^* - L')^2}{L(q)(1-L(q))}.
\end{align}
For $q \in (0, 1)$, $L(q) \geq (1-q)L'$, and $1 - L(q) \geq 1- L'$, and so $\frac{1}{L(q)(1-L(q))} \leq \frac{1}{(1-q)L'(1-L')}$. Applying this with $(L' - L^*)^2 \leq L'^2$ to \eqref{second_deriv}, we have 
\begin{align}
    J_{\lambda_m}''(q) \geq \frac{1}{1 - q}\left(\frac{\lambda_m}{q} - \frac{mL'}{1 - L'}\right) \geq \frac{1}{1-q}\left(\lambda_m - \frac{mL'}{1-L'}\right).
\end{align}
Since $L' < \frac{0.5 - 0.1L^*}{0.9}$, $L'/(1-L') \leq \frac{0.5}{0.4} = 1.25$. Because $\liminf \lambda_m > 10m$, there exists some $c > 0$ such that $J_{\lambda_m}''(q) > cm$ for all $q \in (0,1)$ for large $m$. So $J_{\lambda_m}$ is strongly convex and any stationary point is the unique global minimizer.

Note that as $q \rightarrow 0^+$, $\TJ_{\lambda_m}(q) \rightarrow -\infty$, and as $q \rightarrow 1^-$, $J_{\lambda_m}(q) \rightarrow \infty$. By the assumption $L' < \frac{0.5 - 0.1L^*}{0.9}$, $L(0.1) < 0.5$, and so $J_{\lambda_m}'(0.1) \leq 0$. Because $J_{\lambda_m}'(q)$ is increasing in $q$, we must have $\qstar \geq 0.1$. In particular, if $\frac{\lambda_m}{m} \rightarrow \infty$, $\qstar \rightarrow \Pi(h_0) = 0.1$. 

Because $L_S(h_0)\xrightarrow{p}L^*$ and $L_S(h_1)\xrightarrow{p}L'$, and the map $(a, b, q) \mapsto H(aq + (1-q)b)$ is continuous on the compact set $[0,1]^3$, 
\begin{align}
    U_m(S) \coloneqq \sup_{q \in [0,1]} \abs{H(\hat{L}(q)) - H(L(q))} \xrightarrow{p} 0.
\end{align}
Set 
\begin{align}
    \Delta_m(S) \coloneqq \sup_{q \in [0,1]} \abs{J_{\lambda_m}(q, S) - \TJ_{\lambda_m}(q)} = mU_m(S).
\end{align}
Then $\Delta_m(S)/m \xrightarrow{p} 0$. By the strong convexity of $J_{\lambda_m}$, 
\begin{align}\label{strong_convexity}
    J_{\lambda_m}(q) \geq  J_{\lambda_m}(\qstar) + \frac{cm}{2}(q - \qstar)^2,
\end{align}
for any $q \in [0,1]$. Since $\qhat$ minimizes $\hat{J}_{\lambda_m}$, we have $\hat{J}_{\lambda_m}(\qhat) \leq \hat{J}_{\lambda_m}(\qstar)$. Thus, 
\begin{align}\label{triangle_ineq}
    J_{\lambda_m}(\qhat) - J_{\lambda_m}(\qstar) \leq (J_{\lambda_m} - \hat{J}_{\lambda_m})(\qhat) - (J_{\lambda_m} - \hat{J}_{\lambda_m})(\qstar) \leq 2\Delta_m.
\end{align}
Combining \eqref{strong_convexity} and \eqref{triangle_ineq}, we have $\abs{\qhat - \qstar}^2 \leq \frac{4\Delta_m(S)}{cm}$ for all large $m$. Since $\Delta_m(S)/m \xrightarrow{p} 0$, we have $\qhat \xrightarrow{p}\qstar$. 

By continuity of $L(q)$ and that $\abs{L(\qhat) - L(\qstar)} \leq \abs{L' - L^*} \abs{\qhat - \qstar}$, we have $L_\source(\Qgreen_{\lambda_m}(S)) = L(\qhat)\xrightarrow{p}L(\qstar)$. This implies that $\mathbb{E} L_\source(\Qgreen_{\lambda_m}(S)) = \mathbb{E}L(\qhat) \rightarrow  L(\qstar)$. In particular, if $\frac{\lambda_m}{m} \rightarrow \infty$, $\qstar \rightarrow \Pi(h_0) = 0.1$. So $\mathbb{E} L_\source(\Qgreen_{\lambda_m}(S)) \rightarrow L(0.1) = 0.1L^* + 0.9L'$.

It is important to note that in the definition of $\Qgreen_{\lambda}$, we also require the distribution $\Qgreen_{\lambda}$ to satisfy $L_S(\Qgreen_{\lambda}) \leq
\frac{1}{2}$ (as in equation \eqref{defn:green}). And we just showed $L_S(\Qgreen_{\lambda_m}) = \hat{L}(\qhat) \xrightarrow{p} L(\qstar)$. Since $L(q) = L' + q(L^* - L')$ is decreasing in $q$ and $\qstar \geq 0.1$, we have $L(\qstar) \leq L(0.1) < 0.5$, and so $\Qgreen_{\lambda_m}$ is the output distribution as $m \rightarrow \infty$. This completes the proof for \autoref{LAMBDA>>M}.
   
\end{proof}

\section{Lower Bounds for the Profile Posterior Learning Rule}\label{appendix_blue}
% We now state the lower bound results for the Profile Posterior Learning Rule:
% \begin{align}\label{defn:blue_appendix}
% \QPL_{\lambda}(S) = \underset{Q}{\text{arg min }} \JPL(Q,S), \text{such that } L_S(Q)\leq 1/2,
% \end{align}
% where 
% \begin{align}
%     \JPL(Q,S) = m\mathbb{E}_{h\sim Q}H(L_S(h)) +  \lambda \KL(Q \Vert \Pi).
% \end{align}

% For any $0<\lambda<\infty$, we show that the worst-case limiting error is lower bounded by the same tempering function $\ell_{\lambda}$:
% \begin{gather}\label{ellblue}
% \ell_{\lambda}(L^*)=\removed{ Q_{\lambda}^{-1}(H(L^*))=}
% \begin{cases}
% 1 - 2^{-\frac{1}{\lambda}H(L^*)},  & \text{for } 0< \lambda \leq 1 \\
% U_{\lambda}^{-1}(H(L^*)), & \text{for } \lambda > 1,
% \end{cases} 
% \quad \textrm{where: } { U_{\lambda}(q) = \lambda \KL\!\left(\tfrac{1}{1+ \left(\frac{\scriptscriptstyle 1-q}{\scriptscriptstyle q}\right)^{\frac{\scriptscriptstyle \lambda}{\scriptscriptstyle \lambda - 1}}}\middle\Vert q\right) \!+\! H\!\left(\tfrac{1}{1+ \left(\frac{\scriptscriptstyle 1-q}{\scriptscriptstyle q}\right)^{\frac{\scriptscriptstyle \lambda}{\scriptscriptstyle \lambda - 1}}}\right)}.
% \end{gather}

% $\Qblue_{\lambda_m}$ at $\lambda = 0$ or $\lambda_m \rightarrow 0$ exhibits catastrophic overfitting with the limiting error 1:

% On the other direction, when $\lambda_m = \Omega(m)$, $\Qblue_{\lambda_m}$ leads to catastrophic underfitting:

% \subsection{Lower Bound Constructions and Proof Sketch}
In this Section, we describe  constructive lower bound proofs on the limiting error.  We show explicit constructions for $0 < \lambda < \infty$ ( \autoref{B:LBFIXED}), for $\lambda_m \rightarrow 0$ or $\lambda = 0$ (\autoref{B:LBLAMBDA0}), and for $\lambda_m = \Omega(m)$ with $\liminf \frac{\lambda_m}{m}> 10$ (\autoref{B:LBLAMBDAINF}).  In each regime, we construct specific hard learning problems, priors, and hypothesis classes such that the expected error of $\Qblue_{\lambda}$ converges to the lower bound error asymptotically. 

\fbox{\parbox{0.99\linewidth}{
\textbf{Comparison with Lower Bounds for the PAC-Bayes (\autoref{AppendixB})}
The lower bound proof for the Profile Posterior is almost the same as that of the PAC-Bayes learning rule, except that we do not need to use the concavity of entropy.
}}

\subsection{Lower Bound for $0 < \lambda < \infty$ (proof of \autoref{B:LBFIXED})}

% For any $0<\lambda<\infty$, any $0<L^*<0.5$, and any $L^* \leq L' < \ell_\lambda(L^*)$, we will construct a source distribution (hard learning problem) $D$ and a prior $\Pi$, and show a distribution $Q^*$ with $\KL(Q^* \Vert \Pi) \leq 10$ and $L_D(Q^*) = L^*$, such that $\mathbb{E}_S \left[L_D(\Qblue_{\lambda}(S))\right] \rightarrow L'$ as the sample size increases ($m \rightarrow \infty$). 

\restate{Theorem}{B:LBFIXED} 
\begin{restatethm}[Agnostic Lower Bound]
    For any $0<\lambda<\infty$, any $L^* \in (0,0.5)$ and $L^* \leq L' <\ell_{\lambda}(L^*)$, % there exists a prior $\Pi$, a distribution $Q^*$ with $\KL(Q^* \Vert \Pi) \leq 10$ and source distribution $D$ with $L_\source(Q^*) = L^*$ such that 
    there exists $(\Pi,\source)\in \achieve(L^*)$ such that
    $\mathbb{E}_S \left[L_\source(\Qblue_{\lambda})\right] \rightarrow L'$ as sample size $m \rightarrow \infty$. 
\end{restatethm}

We use the same construction as in \autoref{subsection:B.1}. We will ensure that as $m\rightarrow\infty$, $\Qblue_{\lambda}$ will place zero mass on $h_0$, i.e.~$\Pr_{S\sim D^m}\left[\Qblue_\lambda(h_0) > 0\right]\xrightarrow{m\rightarrow\infty} 0$ and $L(\Qblue_\lambda(S))\xrightarrow{p} L'$.

Let $Q^* = \delta_{h_0}$ be the point mass on $h_0$, and then $L_D(Q^*) = L^*$, and $\KL(Q^* \Vert \Pi) = -\log \Pi(h_0) = \log 10 < 10$. For the ``good'' predictor $h_0$ we have that $L_S(h_0)\xrightarrow{p} L(h_0)=L^*$, and hence $\JPL(Q^*, S) \xrightarrow{p} mH(L^*) + \lambda \log 10 $. For an explicit function $k(m)$, we will show that, with probability approaching one, there exists $1\leq i\leq k(m)$ with $\JPL(\delta_{h_i},S) < \JPL(\delta_{h_0},S)- \Omega(m)$, ensuring $\Qblue_{\lambda}$ will put zero mass (in the limit) on $h_0$.

\begin{enumerate}
% \item $\lambda \leq 1$: Take $k(m) = \frac{2\sqrt{m}}{(1-L')^m}$, then (with probability approaching one), there exists some ``bad'' classifier $h_{\hat{i}}$ with $1\leq \hat{i} \leq k(m)$ such that $L_S(h_{\hat{i}}) = 0$. If $\Qblue_{\lambda}$ puts positive mass on $h_0$, i.e, $\Qblue_{\lambda}(h_0) > 0$, then 
% \begin{align}
%  \JPL(\Qblue_{\lambda},S)
%  &\geq mH(L^*)\Qblue_{\lambda}(h_0) + \lambda \KL(\Qblue_\lambda\Vert \Pi) \geq mH(L^*)\Qblue_{\lambda}(h_0)\\ 
%  &>- \lambda m \log(1-L')\Qblue_{\lambda}(h_0)
%  = mH(0)\Qblue_{\lambda}(h_0) + \lambda (-\log \Pi(h_{\hat{i}}) )\Qblue_{\lambda}(h_0) + O(1)
% \end{align}
% where in the final inequality we used $L' < 1 - 2^{-H(L^*)/{\lambda}}$, and the asymptotic notation is w.r.t.~$m\rightarrow\infty$.

% Therefore, we can always replace the mass on $h_0$ by $h_{\hat{i}}$, and will achieve a lower objective value. This means $\Qblue_{\lambda}$ places all its mass on predictors with expected error $L'$ as $m \rightarrow \infty$.

 \item $\lambda \leq 1$: Choose $k(m) = \frac{2\sqrt{m}}{(1-L')^m}$, then (with probability approaching one), there exists some ``bad'' classifier $h_{\hat{i}}$ with $1\leq \hat{i} \leq k(m)$ such that $L_S(h_{\hat{i}}) = 0$ (proved in B.1.1 in \citet{ZS}). Consider the point mass on the predictor $h_{\hat{i}}$ as $\delta_{h_{\hat{i}}}$. Then we have $\JPL(\delta_{h_{\hat{i}}}, S) = \lambda(-\log \Pi(h_{\hat{i}}))$. Hence, the objective of the minimizer $\Qblue_{\lambda}$ should be upper bounded by the objective of $\delta_{h_{\hat{i}}}$, i.e., $\JPL(\Qblue_{\lambda}, S) \leq \lambda(-\log \Pi(h_{\hat{i}}))$ as $m \rightarrow \infty$.

    Suppose $\Qblue_{\lambda}(h_0) = q_0 > 0$ and $\Qblue_\lambda(h_{\hat{i}}) =  1- q_0$ for the sake of contradiction. For the ``good'' predictor $h_0$ we have that $L_S(h_0)\xrightarrow{p} L_\source(h_0)=L^*$.  By the definition of $\Pi$ and that $h_{\hat{i}} \in \mathcal{H}_{k(m)}$, we have $-\log \Pi(h_{\hat{i}}) \leq m\log \frac{1}{1-L'} + C\log m$, for some constant $C>0$. So    
    \begin{align}
        &\JPL(\Qblue_{\lambda}, S) = m\mathbb{E}_{h \sim \Qblue_{\lambda}}H(L_S(h)) + \lambda \KL(\Qblue_{\lambda} \Vert \Pi)
        = mq_0H(L_S(h_0)) + \lambda(1 - q_0)\log \frac{1}{\Pi(h_{\hat{i}})} + O(1) \\
        &\xrightarrow{p} mq_0H(L^*) + \lambda(1 - q_0)\log \frac{1}{\Pi(h_{\hat{i}})} + O(1)
        > -mq_0 \lambda \log (1-L') + \lambda(1 - q_0)\log \frac{1}{\Pi(h_{\hat{i}})} + \Omega(m)\label{ineq54}\\
        &\geq q_0 \lambda \log \frac{1}{\Pi(h_{\hat{i}})} + \lambda(1 - q_0)\log \frac{1}{\Pi(h_{\hat{i}})}+ \Omega(m)
        = \JPL(\delta_{h_{\hat{i}}}, S) + \Omega(m),
    \end{align}
    where in the inequality of \eqref{ineq54} we used $L' < 1 - 2^{-H(L^*)/{\lambda}}$, and the asymptotic notation is w.r.t.~$m\rightarrow\infty$. This leads to a contradiction. Thus, as long as $L' < 1 - 2^{-H(L^*)/{\lambda}}$, as $m \rightarrow \infty$, the mass on $h_0$ should go to zero with probability one, which implies the bound for the expected risk:
    as $m \rightarrow \infty$, $\mathbb{E}_S \left[L_\source(\Qgreen_{\lambda}(S))\right] \rightarrow L'$.  This completes the proof for $ 0< \lambda \leq 1$.

% \item $\lambda>1$: Take $k(m) = 2^{m\KL(\hat{L}\Vert L')}$ where $\hat{L} = \frac{1}{1+ (\frac{1-L'}{L'})^\frac{\lambda}{\lambda - 1}}$. Let $h_{\hat{i}}$ be the empirical error minimizer among the first $k(m)$ bad predictors, i.e.~such that $L_S(h_{\hat{i}}) = \min_{i=1\ldots k(m)} L_S(h_i)$.  This is the minimum of $k(m)$ independent (scaled) binomials $\text{Bin}(m,L')$, and so concentrates s.t.~$\KL( L_S(h_{\hat{i}}) \Vert L') \xrightarrow{p} \frac{\log k(m)}{m}=\KL(\hat{L}\Vert L')$, and hence $L_S(h_{\hat{i}}) \xrightarrow{p} \hat{L}$. If $\Qblue_{\lambda}(h_0) > 0$, then
% \begin{align}
%    \JPL(\Qblue_{\lambda},S) &\geq mH(L^*)\Qblue_{\lambda}(h_0) + \lambda \KL(\Qblue_\lambda\Vert \Pi) \geq mH(L^*)\Qblue_{\lambda}(h_0)\\
%    &\geq m U_\lambda(L')\Qblue_{\lambda}(h_0) + o(m) = m\left(\lambda \KL(\hat{L}\Vert L')+H(\hat{L})\right)\Qblue_{\lambda}(h_0) + o(m) \label{eq47blue}\\
%    &= (mH(L_S(h_{\hat{i}})) + \lambda(-\log \Pi(h_{\hat{i}})))\Qblue_{\lambda}(h_0)
% \end{align}
% where we plugged in $k(m)$ and used the definition of $U_\lambda$ from equation \eqref{ellblue}, and in \eqref{eq47blue} we relied on  $L' < U_{\lambda}^{-1} ( H(L^*) )$.

% Therefore, we can always replace the mass on $h_0$ by $h_{\hat{i}}$, and will achieve a lower objective value. This means $\Qblue_{\lambda}$ places all its mass on predictors with expected error $L'$ as $m \rightarrow \infty$.\qed

\item $\lambda > 1$: Choose $k(m) = 2^{m\KL(\hat{L}\Vert L')}$ where $\hat{L} = \frac{1}{1+ (\frac{1-L'}{L'})^\frac{\lambda}{\lambda - 1}}$. Let $h_{\hat{i}}$ be the empirical error minimizer among the first $k(m)$ bad predictors, i.e.~such that $L_S(h_{\hat{i}}) = \min_{i=1\ldots k(m)} L_S(h_i)$.  This is the minimum of $k(m)$ independent (scaled) binomials $\text{Bin}(m,L')$ so concentrates s.t.~$\KL( L_S(h_{\hat{i}}) \Vert L') \xrightarrow{p} \frac{\log k(m)}{m}=\KL(\hat{L}\Vert L')$(by Theorem 1 in \citet{binomials}), and hence $L_S(h_{\hat{i}}) \xrightarrow{p} \hat{L}$ (proved in B.1.2 in \citet{ZS}). Consider the point mass on the predictor $h_{\hat{i}}$ as $\delta_{h_{\hat{i}}}$. Then we have $\JPL(\delta_{h_{\hat{i}}}, S) = mH(L_S(h_{\hat{i}})) +  \lambda(-\log \Pi(h_{\hat{i}}))$. Hence, the objective of the minimizer $\Qblue_{\lambda}$ should be upper bounded by the objective of $\delta_{h_{\hat{i}}}$, i.e., $\JPL(\Qblue_{\lambda}, S) \leq mH(L_S(h_{\hat{i}})) + \lambda(-\log \Pi(h_{\hat{i}}))$ as $m \rightarrow \infty$.
    
    Suppose $\Qblue_\lambda(h_0) = q_0 > 0$ and $\Qblue_\lambda(h_{\hat{i}}) =  1- q_0$ for the sake of contradiction.  By the definition of $\Pi$ and that $h_{\hat{i}} \in \mathcal{H}_{k(m)}$, we have $-\log \Pi(h_{\hat{i}}) \leq m\KL(\hat{L}\Vert L') + C'\log m$, for some $C' > 0$. Then 
    \begin{align}
        &\JPL(\Qblue_{\lambda}, S) = m\mathbb{E}_{h \sim \Qblue_{\lambda}}H(L_S(h)) + \lambda \KL(\Qgreen_{\lambda} \Vert \Pi)\\
        &= m\left(q_0H(L_S(h_0)) + (1-q_0)H(L_S(h_{\hat{i}}))\right)
        + \lambda(1 - q_0)\log \frac{1}{\Pi(h_{\hat{i}})} + O(1)\\
        &\xrightarrow{p} mq_0H(L^*) + (1-q_0)\Jgreen(\delta_{h_{\hat{i}}}, S) + O(1)\\
        &> mq_0U_{\lambda}(L') + (1-q_0)\Jgreen(\delta_{h_{\hat{i}}}, S)+ \Omega(m) \label{ineq57}\\
        &= mq_0\left(\lambda \KL(\hat{L}\Vert L')+H(\hat{L})\right)
        + (1-q_0)\Jgreen(\delta_{h_{\hat{i}}}, S)+ \Omega(m)\\
        &\geq q_0\left(\lambda (-\log \Pi(h_{\hat{i}}))+mH(L_S(h_{\hat{i}}))\right)
        + (1-q_0)\Jgreen(\delta_{h_{\hat{i}}}, S)+ \Omega(m)\\       &>q_0\Jgreen(\delta_{h_{\hat{i}}},S) + (1-q_0)\Jgreen(\delta_{h_{\hat{i}}}, S)+ \Omega(m)
        = \Jgreen(\delta_{h_{\hat{i}}},S)+ \Omega(m),
    \end{align}
where in \eqref{ineq57} we relied on  $L' < U_{\lambda}^{-1} ( H(L^*) )$ and the definition of $U_\lambda$, and the asymptotic notation is w.r.t.~$m\rightarrow\infty$. This leads to a contradiction. Hence, as long as  $L' < U_{\lambda}^{-1} ( H(L^*) )$, as $m \rightarrow \infty$, the mass on $h_0$ will go to zero with probability one. 

This implies the bound for the expected risk:
as $m \rightarrow \infty$, $\mathbb{E}_S \left[L_\source(\Qblue_{\lambda}(S))\right] \rightarrow L'$. This completes the proof for $ 1< \lambda < \infty$. 
\end{enumerate}

\subsection{Lower Bound for $\lambda_m \rightarrow 0$ or $\lambda = 0$ (proof of \autoref{B:LBLAMBDA0})}

\restate{Theorem}{B:LBLAMBDA0} 
\begin{restatethm}
    For any $\lambda_m \rightarrow 0$ or $\lambda = 0$, any $L^* \in (0,0.5)$, and $L^* \leq L' < 1$, %there exists a prior $\Pi$, a distribution $Q^*$ with $\KL(Q^* \Vert \Pi) \leq 10$ and source distribution $D$ with $L_\source(Q^*) = L^*$ such that 
    there exists $(\Pi,\source)\in \achieve(L^*)$ such that
    $\mathbb{E}_S \left[L_\source(\Qblue_{\lambda})\right] \rightarrow L'$ as sample size $m \rightarrow \infty$.
\end{restatethm}
% For $\lambda_m \rightarrow 0$ or $\lambda = 0$, we show that for any $0<L^* <0.5$ and $L^* \leq L' < 1$, the source distribution and prior defined in \autoref{subsection:B.1} such that $\mathbb{E}_S \left[L_\source(\Qblue_{\lambda}(S))\right] \rightarrow L'$, as $m \rightarrow \infty$ while $L(Q^*)=L^*$ and $\KL(Q^* \Vert \Pi) \leq 10$.

For $Q^*$ being a point mass on $h_0$, we have $L_D(Q^*) = L^*$ and $\KL(Q^* \Vert \Pi) = \log 10 < 10$.

For $\lambda = 0$, then the learning rule simply minimizes $\mathbb{E}_{h \sim Q}H(L_S(h))$. There exists a.s.~some $\hat{i}$ with $L_S(h_{\hat{i}})=0$, but on the other hand $L_S(h_0)\xrightarrow{p}L^*>0$. Consider $Q'$ being a point mass on $h_{\hat{i}}$, then we have $m\mathbb{E}_{h \sim Q'}H(L_S(h)) +\lambda \KL(Q'\Vert \Pi) = 0$. So by definition of $\Qblue_{\lambda}$, we must have $m\mathbb{E}_{h \sim \Qblue_{\lambda}}H(L_S(h)) +\lambda \KL(\Qblue_{\lambda}\Vert \Pi) = m\mathbb{E}_{h \sim \Qblue_{\lambda}}H(L_S(h)) \leq 0$. Any $Q$ with a positive mass on $h_0$, i.e. $Q(h_0) > 0$ must have $m\mathbb{E}_{h \sim Q}H(L_S(h)) > mH(L_S(h_0))Q(h_0) \xrightarrow{p} mH(L^*)Q(h_0) > 0$, which is a contradiction. Hence, as $m \rightarrow \infty$, $\Qblue_{\lambda}(h_0)$ should go to zero with probability approaching one, which implies the bound for the expected risk: as $m \rightarrow \infty$, $\mathbb{E}_S \left[L_\source(\Qblue_{\lambda}(S))\right] \rightarrow L'$.  

For $\lambda_m \rightarrow 0$ as $m \rightarrow \infty$, let $\hat{i}$ denote the smallest index $\hat{i}\geq 1$ such that $L_S(h_{\hat{i}}) = 0$. We already saw that $\hat{i} \leq\frac{m+1}{(1-L')^m}$ with probability approaching one. Consider $Q'$ being a point mass on $h_{\hat{i}}$, then we have $m\mathbb{E}_{h \sim Q'}H(L_S(h)) +\lambda_m \KL(Q'\Vert \Pi) \leq\lambda_m \log\frac{m+1}{(1-L')^m}=o(m)$ as $\lambda_m \rightarrow 0$. So by definition of $\Qblue_{\lambda}$, we must have $m\mathbb{E}_{h \sim \Qblue_{\lambda}}H(L_S(h)) +\lambda_m \KL(\Qblue_{\lambda}\Vert \Pi) \leq o(m)$. Any $Q$ with a positive mass on $h_0$, i.e. $Q(h_0) > 0$ must have $m\mathbb{E}_{h \sim Q}H(L_S(h)) + \lambda_m \KL(Q\Vert \Pi) > mH(L_S(h_0))Q(h_0) \xrightarrow{p} mH(L^*)Q(h_0) = O(m)$, which is a contradiction. Hence, with probability approaching one, $\Qblue_{\lambda}(h_0)$ should go to zero, and so $\mathbb{E}_S \left[L_\source(\Qblue_{\lambda}(S))\right] \rightarrow L'$.   \qed

\subsection{Lower Bound for $\lambda_m = \Omega(m)$ with $\liminf \frac{\lambda_m}{m}> 10$ (proof of \autoref{B:LBLAMBDAINF})}
% We now turn to $\lambda_m = \Omega(m)$ with $\liminf \frac{\lambda_m}{m}> 10$, and show that for any $0 \leq L^* <0.5$ and $L^* \leq L' < \frac{0.5 - 0.1L^*}{0.9}$, and for $q^* = \arg \min _{q}\lambda_m\KL(q \Vert 0.1) + mqH(L^*) + m(1-q)H(L')$. 
\restate{Theorem}{B:LBLAMBDAINF} 
\begin{restatethm}
For any $\lambda_m = \Omega(m)$ with $\liminf \frac{\lambda_m}{m}> 10$, any $L^* \in [0,0.5)$, and any $L^* \leq L' < \frac{0.5 - 0.1L^*}{0.9}$, and for $q^*_m =\sigma\left(\text{logit}(0.1) + \frac{m}{\lambda_m}(H(L') - H(L^*))\right)$ where $\sigma(t) = 1/(1+e^{-t})$, %there exists a prior $\Pi$, a distribution $Q^*$ with $\KL(Q^* \Vert \Pi) \leq 10$ and source distribution $D$ with $L_\source(Q^*) = L^*$ such that 
there exists $(\Pi,\source)\in \achieve(L^*)$ such that
$\mathbb{E}_S \left[L_\source(\Qblue_{\lambda_m}(S))\right] = q^*_mL^* + (1-q^*_m)L' + o(1)$. In particular, if $\frac{\lambda_m}{m} \rightarrow \infty$, then $q^*_m \rightarrow 0.1$ and $\mathbb{E}_S \left[L_\source(\Qblue_{\lambda_m}(S))\right] \rightarrow 0.1L^* + 0.9L'$ as $m \rightarrow \infty$.
\end{restatethm}
We use the source distribution defined in \autoref{subsection:B.1} with only two predictors $\{h_0, h_1\}$, $L(h_0)=L^*,L(h_1)=L'$, and with the prior $\Pi(h_0)=0.1$ and $\Pi(h_1)= 0.9$.

For $Q^*$ being a point mass on $h_0$, we have $L_D(Q^*) = L^*$ and $\KL(Q^* \Vert \Pi) = \log 10 < 10$.

Let $\Qblue_{\lambda_m}$ be the output distribution, and let $\Qblue_{\lambda_m}(h_0) = q$, $\Qblue_{\lambda_m}(h_1) = 1-q$. Then
\begin{align}
        \Jbluem(\Qblue_{\lambda_m}, S)=\lambda_m\KL(q \Vert 0.1) + m\mathbb{E}_{h \sim \Qblue_{\lambda_m}}H(L_S(h))
        = \lambda_m\KL(q \Vert 0.1) + mqH(L_S(h_0)) + m(1-q)H(L_S(h_1)).
\end{align}
Taking the derivative of the right hand side we get 
\begin{gather}
    \frac{d}{dq}\Jbluem(q) = \lambda_m \left(\log\frac{q}{1-q} - \log \frac{0.1}{1-0.1}\right) + m(H(L_S(h_0)) - H(L_S(h_1))), \textrm{ and } \frac{d^2}{dq^2}\Jbluem(q) = \frac{\lambda_m}{q(1-q)} > 0.
\end{gather}
So we have a unique minimizer 
\begin{align}
    \hat{q}_m(S) = \sigma\left(\text{logit}(0.1) + \frac{m}{\lambda_m}(H(L_S(h_1)) - H(L_S(h_0)))\right),
\end{align}
with $\sigma(t) = 1/(1+e^{-t})$.

Because $L_S(h_0)\xrightarrow{p}L^*$ and $L_S(h_1)\xrightarrow{p}L'$, and $H$ is continuous, $H(L_S(h_1)) - H(L_S(h_0)) \xrightarrow{p} H(L^*) - H(L')$. By continuity of $\sigma$, 
\begin{align}
    \qhat(S) \xrightarrow{p} q^*_m \coloneqq \sigma\left(\text{logit}(0.1) + \frac{m}{\lambda_m}(H(L') - H(L^*)\right) \geq 0.1.
\end{align}

Denote $L(q) = qL^* + (1-q)L'$. By continuity of $L(q)$ and that $\abs{L(\qhat) - L(\qstar)} \leq \abs{L' - L^*} \abs{\qhat - \qstar}$, we have $L_\source(\Qblue_{\lambda_m}(S)) = L(\qhat)\xrightarrow{p}L(\qstar)$. This implies that $\mathbb{E} L_\source(\Qblue_{\lambda_m}(S)) = \mathbb{E}L(\qhat) \rightarrow  L(\qstar)$. In particular, if $\frac{\lambda_m}{m} \rightarrow \infty$, $\qstar \rightarrow \Pi(h_0) = 0.1$. So $\mathbb{E} L_\source(\Qblue_{\lambda_m}(S)) \rightarrow L(0.1) = 0.1L^* + 0.9L'$.

It is important to note that in the definition of $\Qblue_{\lambda}$, we also require the distribution $\Qblue_{\lambda}$ to satisfy $L_S(\Qblue_{\lambda}) \leq
\frac{1}{2}$ (as in equation \eqref{defn:blue}). Note that we just showed $L_S(\Qblue_{\lambda_m}) = \qhat(S) L_S(h_0) + (1-\qhat(S))L_S(h_1) \xrightarrow{p} L(\qstar)$. Since $L(q) = L' + q(L^* - L')$ is decreasing in $q$ and $\qstar \geq 0.1$, we have $L(\qstar) \leq L(0.1) < 0.5$, and so $\Qblue_{\lambda_m}$ is the output distribution as $m \rightarrow \infty$. \qed

\section{Bayesian Interpretation Proofs}

We now present the remaining proofs of the Bayesian interpretations of the learning rules from \autoref{section:EB} and \autoref{section:Bayesian}.

\restate{Theorem}{greenEB:lambda=1}
\begin{restatethm}
    Let $\hat{\eta} = \arg \max_{\eta} P(S|\eta) = \arg \max_{\eta} \mathbb{E}_{h \sim \Pi}P(S|\eta,h)$. Then $\QEB_{1}(h) \equiv P(h|S, \hat{\eta})$.
\end{restatethm}
\begin{proof}
    The likelihood of the model is defined as $P(S|\eta, h) = \eta^{mL_S(h)}(1-\eta)^{m(1 - L_S(h))} = (1-\eta)^m 2^{-mH'(\eta)L_S(h)} \propto 2^{-mH'(\eta)L_S(h)}$. $P(S|\eta) = \mathbb{E}_{h \sim \Pi} P(S|\eta, h) = (1-\eta)^m \mathbb{E}_{h \sim \Pi}[2^{mL_S(h)\theta}]$, where $\theta = \log \frac{\eta}{1 - \eta}$. Denote $Z(\eta) = \mathbb{E}_{h \sim \Pi}[2^{mL_S(h)\theta}]$. To maximize $P(S|\eta)$ over $\eta$, we take the logarithm and set its derivative to zero, and get $\frac{d}{d\eta} \log P(S|\eta) = -\frac{m}{1 - \eta} + \frac{d}{d\theta}(\log Z(\eta))\frac{d\theta}{d\eta} = 0$. Note that $\frac{d\theta}{d\eta} = \frac{1}{\eta(1-\eta)}$, and $\frac{d}{d\theta}\log Z(\eta) = \frac{1}{Z(\eta)}\mathbb{E}_{h \sim \Pi}[\frac{d}{d\theta}2^{m\theta L_S(h)}] = m \int \frac{2^{m\theta L_S(h)}\Pi(h)}{Z(\eta)}L_S(h) dh = m\mathbb{E}_{h \sim P(h|S, \eta)}[L_S(h)]$. Plugging it in yields that $\hat{\eta} = \mathbb{E}_{h \sim P(h|S, \hat{\eta})}[L_S(h)]$. Therefore, $P(h|S, \hat{\eta}) \propto \Pi(h) P(S|\hat{\eta}, h) \propto 2^{-mH'(L_S(P(h|S, \hat{\eta})))L_S(h)}\Pi(h)$. Comparing with equation \eqref{Gibbs_green}, it is easy to see that $\QEB(h) \equiv P(h|S, \hat{\eta})$. 
\end{proof}

\restate{Theorem}{greenEB:m_prior} 
\begin{restatethm}
    Consider the prior $P_{\lambda}(\eta) \propto ((1-\eta)^\lambda + \eta^\lambda)^{-\frac{m}{\lambda}}$. Let $\hat{\eta} = \arg \max_{\eta} P(S|\eta)P_{\lambda}(\eta) = \arg \max_{\eta} \mathbb{E}_{h \sim \Pi}P(S|\eta,h)P_{\lambda}(\eta)$. Then $\QEB_{\lambda}(h) \equiv P(h|S, \hat{\eta}) \propto \Pi(h) P(S|\hat{\eta}, h) $.
\end{restatethm}
\begin{proof}
    Note that $\frac{d}{d\eta} \left(\log P(S|\eta) + \log P_{\lambda}(\eta) \right) = -\frac{m}{1 - \eta} + mL_S(P(h|S, \eta))\frac{1}{\eta(1-\eta)} + \frac{d}{d\eta}\log P_{\lambda}(\eta) = 0$, which follows from the proof of \autoref{greenEB:lambda=1}.
    
   Using $P_{\lambda}(\eta) \propto ((1-\eta)^\lambda + \eta^\lambda)^{-\frac{m}{\lambda}}$, we get that
   \[
        \frac{d}{d\eta} \log P_{\lambda}(\eta) = \frac{m}{1-\eta}+m\frac{1}{\eta(1-\eta)}\frac{e^{\lambda\theta}}{1+e^{\lambda\theta}},
   \]
    where we again use the reparametrization with $\theta=\log \frac{\eta}{1-\eta}$. Substituting into the optimality condition above and solving for $e^{\lambda\theta}$, we obtain that the maximizer $\hat{\eta}$ satisfies
    \[
        \frac{1 - \hat{\eta}}{\hat{\eta}} = \left(\frac{1 - L_S(P(h|S, \hat{\eta}))}{L_S(P(h|S, \hat{\eta}))}\right)^{\frac{1}{\lambda}}.
    \]
    Thus, $P(S|\hat{\eta}, h) \propto 2^{-m\log \frac{1 - \hat{\eta}}{\hat{\eta}} L_S(h)} = 2^{-\frac{m}{\lambda}\log \frac{1 - L_S(P(h|S, \hat{\eta}))}{L_S(P(h|S, \hat{\eta}))} L_S(h)} = 2^{-\frac{m}{\lambda}H'( L_S(P(h|S, \hat{\eta}))) L_S(h)}$. Therefore, $P(h|S, \hat{\eta}) \propto \Pi(h) P(S|\hat{\eta}, h) \propto 2^{-\frac{m}{\lambda}H'(L_S(P(h|S, \hat{\eta})))L_S(h)}\Pi(h)$, implying that $\QEB_{\lambda}(h) \equiv P(h|S, \hat{\eta})$. 
\end{proof}

\restate{Theorem}{greenEB:fixed_prior}
\begin{restatethm}
    For any fixed prior $P_{\lambda}(\eta)$ which does not depend on $m$, $\frac{dP(h|S, \hat{\eta})}{d \QEB_{1}(h)} \rightarrow 1$ uniformly as $m \rightarrow \infty$.
\end{restatethm}
\begin{proof}
    The first order condition gives us $ -\frac{m}{1 - \eta} + mL_S(P(h|S, \eta))\frac{1}{\eta(1-\eta)} + \frac{d}{d\eta}\log P_{\lambda}(\eta) = 0$. A first-order expansion gives us that $\hat{\eta} = L_S(P(h|S, \eta)) + O(\frac{1}{m})$. Plugging this, by Taylor expanding $H'$ at $L_S(P(h|S, \eta))$ we have $P(S|\hat{\eta}, h) \propto 2^{-mH'(\hat{\eta}) L_S(h)} = 2^{-mH'( L_S(P(h|S, \hat{\eta}))) L_S(h)}2^{-O(1)}$. %\redcomment{We should probably iron this proof out a bit more, to make sure we bound both the numerators and the denominators/normalizations.}
\end{proof}

\restate{Theorem}{PL_lambda_prior}
\begin{restatethm}
    Consider the prior $P_{\lambda}(\eta) \propto ((1-\eta)^\lambda + \eta^\lambda)^{-\frac{m}{\lambda}}$. Let $\hat{\eta}_h = \arg \max_{\eta} P(S \mid \eta,h)P_{\lambda}(\eta)$. Then
    \begin{align*}
    \QPL_{\lambda}(h) &\propto \Pi(h)P(S \mid \hat{\eta}_h, h)P_{\lambda}(\hat{\eta}_h)
    \\ 
    &\propto \max_\eta \Pi(h)P(S|\eta,h)P_\lambda(\eta).
    \end{align*}
\end{restatethm}
\begin{proof}
    Let $\hat{\eta}_h = \arg \max_\eta \log P(S\mid \eta, h) + \log P_{\lambda}(\eta) = m\left(L_S(h)\log \eta + (1-L_S(h)) \log (1-\eta)\right) - \frac{m}{\lambda}\log (\eta^{\lambda} + (1-\eta)^{\lambda})$. To take the derivative, we proceed similarly to the proof of \autoref{greenEB:m_prior}, and obtain that the optimal $\hat\eta_h$ satisfies:
    \[
        \frac{1 - \hat{\eta}_h}{\hat{\eta}_h} = \left(\frac{1 - L_S(h))}{L_S(h)}\right)^{\frac{1}{\lambda}}.
    \]    
    % Taking the derivative and setting it to zero, we get $\hat{\eta}_h = \frac{(L_S(h))^{\frac{1}{\lambda}}}{(L_S(h))^{\frac{1}{\lambda}} + (1-L_S(h))^{\frac{1}{\lambda}}}$.
    Plugging this in, a few manipulations yield that $P(S | \hat{\eta}_h, h)P_{\lambda}(\hat{\eta}_h) = \left((1-L_S(h))^{1 - L_S(h)}L_S(h)^{L_S(h)}\right)^{\frac{m}{\lambda}} = 2^{-\frac{m}{\lambda}H(L_S(h))}$. So we have $\Pi(h)P(S | \hat{\eta}_h, h)P_{\lambda}(\hat{\eta}_h) = \Pi(h)2^{-\frac{m}{\lambda}H(L_S(h))} \propto \QPL_{\lambda}(h)$.
\end{proof}

\restate{Theorem}{thm:mod-bayes-equivalence-lambda-1}
\begin{restatethm}
$\frac{d\QBayes_1(h)}{d \QPL_{1}(h)} \rightarrow 1$ uniformly as $m \rightarrow \infty$.
\end{restatethm}  
\begin{proof}
    $P(h \mid S) \propto \Pi(h) I(h)$, where $I(h) = \int P(S \mid \eta, h)P_{\lambda}(\eta) d\eta$. $\log I(h) = \log \int \exp \{m\Psi_{\lambda}(\eta, L_S(h))\}d\eta$, where $\Psi_{\lambda}(\eta, L_S(h)) = (1-L_S(h))\log (1-\eta) + L_S(h)\log \eta$. Then by Laplace's method, $\log I(h) = \log \int \exp \{m\Psi_{\lambda}(\eta, L_S(h))\}d\eta =  m\max_{\eta} \Psi_{\lambda}(\eta, L_S(h)) + O(\log m)$. And we can check that $\max_{\eta} \Psi_{\lambda}(\eta, L_S(h)) = -H(L_S(h))$. Hence $\QBayes_1 = P(h|S) \propto \Pi(h)2^{-mH(L_S(h))(1+o(1))}$.
\end{proof}

\restate{Theorem}{thm:mod-bayes-equivalence-any-lambda}
\begin{restatethm}
    With the same prior $P_{\lambda}(\eta) = ((1-\eta)^\lambda + \eta^\lambda)^{-\frac{m}{\lambda}}$, $\frac{d\QBayes_\lambda(h)}{d \QPL_{\lambda}(h)} \rightarrow 1$ uniformly as $m \rightarrow \infty$.
\end{restatethm} 
\begin{proof}
    $P(h \mid S) \propto \Pi(h) I(h)$, where $I(h) = \int P(S \mid \eta, h)P_{\lambda}(\eta) d\eta$. $\log I(h) = \log \int \exp \{m\Psi_{\lambda}(\eta, L_S(h))\}d\eta$, where $\Psi_{\lambda}(\eta, L_S(h)) = (1-L_S(h))\log ((1-\eta) + L_S(h)\log \eta - \frac{1}{\lambda}\log (\eta^{\lambda} + (1-\eta)^{\lambda})$. Then by Laplace's method, $\log I(h) = \log \int \exp \{m\Psi_{\lambda}(\eta, L_S(h))\}d\eta =  m\max_{\eta} \Psi_{\lambda}(\eta, L_S(h)) + O(\log m)$. Notice that this is the same objective that $\QPL_{\lambda}$ maximizes in \autoref{PL_lambda_prior}, and we have checked that $\max_{\eta} \Psi_{\lambda}(\eta, L_S(h)) = -\frac{1}{\lambda}H(L_S(h))$. Hence $\QBayes_\lambda = P(h|S) \propto \Pi(h)2^{-\frac{m}{\lambda}H(L_S(h))(1+o(1))}$. 
\end{proof}

\restate{Theorem}{thm:mod-bayes-equivalence-fixed-prior}
\begin{restatethm}
    With any fixed prior that does not depend on $m$, $\frac{d\QBayes(h)}{d \QPL_{1}(h)} \rightarrow 1$ uniformly as $m \rightarrow \infty$.
\end{restatethm}
\begin{proof}
    Denote the fixed prior by $P(\eta)$. Then $P(h \mid S) \propto \Pi(h) \int P(\eta)(1-\eta)^{mL_S(h)}\eta^{mL_S(h)}d\eta$. By Laplace's method, $\int P(\eta)(1-\eta)^{mL_S(h)}\eta^{mL_S(h)}d\eta = \int P(\eta)\exp\{m(1-\eta)\log (1- L_S(h)) + mL_S(h)\log\eta\}d\eta = P(L_S(h))2^{-mH(L_S(h))}(1+o(1))$ as $L_S(h) = \arg \max_\eta (1-\eta)\log (1- L_S(h)) + L_S(h)\log\eta$. Hence, $\QBayes = P(h|S) \propto \Pi(h)2^{-mH(L_S(h))}$. 
\end{proof}

\end{document}